\documentclass[journal,twoside,web]{ieeecolor}

\usepackage{graphics} 
\usepackage{epsfig} 
\usepackage{times} 
\usepackage{amsmath} 
\usepackage{amssymb}  
\usepackage{lineno}
\usepackage{xcolor}
\definecolor{subsectioncolor}{RGB}{0,51,102}
\usepackage{subcaption}
\usepackage{url}
\usepackage{tikz-cd}

\newcommand{\ep}{\varepsilon}

\newcommand{\R}{\mathbb R}
\newcommand{\N}{\mathbb N}

\newcommand{\Q}{\mathbb Q}

\DeclareMathOperator{\supp}{supp}

\DeclareMathOperator{\id}{id}
\DeclareMathOperator{\dom}{dom}
\DeclareMathOperator{\img}{img}

\DeclareMathOperator{\ind}{ind}

\DeclareMathOperator{\co}{co}
\newcommand{\PP}{\tilde{P}}
\newcommand{\dotProduct}[2]{\left\langle #1, #2 \right\rangle}
\DeclareMathOperator{\proj}{proj}

\newtheorem{theorem}{Theorem}
\newtheorem{definition}{Definition}
\newtheorem{problem}{Problem}
\newtheorem{proposition}{Proposition}
\newtheorem{lemma}[theorem]{Lemma}
\newtheorem{example}{Example}

\newtheorem{remark}{Remark}
\newtheorem{assumption}{Assumption}

\title{\LARGE \bf
Efficient Path Generation with Curvature Guarantees by Mollification
}

\author{Alfredo González-Calvin, Juan F. Jiménez Castellanos, and Héctor García 
de Marina
\thanks{Alfredo and Juan are with the Department of Computer Architecture and
Automation, Faculty of Physics, Complutense University of Madrid, Madrid, Spain.
Hector is with the Department of Computer Engineering, Automation, and Robotics
(ICAR) \& Institute of Mathematics (IMAG), University of Granada, Spain. This
work is specially supported by the FPU program of the Ministry of science,
innovation and universities of Spain and it is supported by iRoboCity2030-CM,
Ref TEC-2024/TEC-62, financed by Comunidad Autónoma de Madrid (Spain) and by the
ERC Starting Grant iSwarm 101076091 and the RYC2020-030090-I grant from the
Spanish Ministry of Science. Corresponding author {\tt\small alfredgo@ucm.es}.}
}

\begin{document}

\maketitle
\thispagestyle{empty}
\pagestyle{empty}

\begin{abstract}
Path generation, the process of converting high-level mission specifications, such as sequences of waypoints from a path planner, into smooth, executable paths, is a fundamental challenge in mobile robotics. Most path following and trajectory tracking algorithms require the desired path to be defined by at least twice continuously differentiable functions to guarantee key properties such as global convergence, especially for nonholonomic robots like unicycles with speed constraints. Consequently, path generation methods must bridge the gap between convenient but non-differentiable planning outputs, such as piecewise linear segments, and the differentiability requirements imposed by downstream control algorithms. While techniques such as spline interpolation or optimization-based methods are commonly used to smooth non-differentiable paths or create feasible ones from sequences of waypoints, they either produce unnecessarily complex trajectories or are computationally expensive. In this work, we present a method to regularize non-differentiable functions and generate feasible paths through mollification. Specifically, we approximate an arbitrary path with a differentiable function that can converge to it with arbitrary precision. Additionally, we provide a systematic method for bounding the curvature of generated paths, which we demonstrate by applying it to paths resulting from linking a sequence of waypoints with segments. The proposed approach is analytically shown to be computationally more efficient than standard interpolation methods, enabling real-time implementation on microcontrollers, while remaining compatible with standard trajectory tracking and path following algorithms.
\end{abstract}

\section{Introduction}\label{sec: intro}

In robotics and control, motion execution is often decomposed into a hierarchy of interrelated problems. Path planning determines a feasible route through the robot's configuration space—typically as a sparse sequence of waypoints—using methods ranging from graph-search algorithms like A* to sampling-based approaches such as RRT \cite{lavalle2006planning}. Path generation then transforms this abstract plan into a smooth, continuous geometric path or time-parameterized trajectory, often using Spline-based methods, including cubic B-splines~\cite{lau2009kinodynamic, berglund2010planning} and B\'ezier
curves~\cite{yang2010analytical}. These methods ensure $C^2$ continuity and bounded curvature but may produce high curvature and cumbersome paths when the spacing of waypoints is irregular \cite{meek1992}. Gaussian process
regression~\cite{rasmussen2006} and kernel smoothing~\cite{hastie2009} generate $C^\infty$-smooth approximations, trading exact waypoint passage for improved smoothness at the cost of $O(n^3)$ computational complexity and potential numerical instability. Optimization-based approaches~\cite{mellinger2011, ratliff2009chomp} incorporate explicit curvature bounds through convex formulations, enabling real-time computation but requiring careful parameter tuning~\cite{heiden2018grips}. 

However, critically, path generation does not operate in isolation: its output must meet the requirements imposed by the downstream control layer, which typically takes one of two forms. In \textit{path following}, the generated path must be geometrically smooth and continuously differentiable, so that a well-defined tangent and curvature exist at every point for spatial error computation \cite{weijiaarticlegvf}. In \textit{trajectory tracking}, the demands are more stringent: the time parameterization must additionally be consistent with the robot's dynamic capabilities, as a reference trajectory that demands accelerations or velocities beyond the system's limits cannot be reliably tracked regardless of the controller's design \cite{siciliano2009robotics}. In both cases, path generation must internalize the robot's kinematic and dynamic constraints, effectively acting as the bridge that makes a planned route physically executable and controllable.

In this paper, we propose an inexpensive path generation method based on \emph{mollifying} non-differentiable paths, e.g., piecewise functions, using mollifier functions. Mollifiers are smooth functions that, via convolution, approximate non-smooth functions arbitrarily close. They have been extensively used in partial  differential equations~\cite{Evans2022-PDE} and in studying the non-vanishing of generalized Riemann zeta functions~\cite{cech2025optimalitymollifiers}. From an  engineering perspective, mollifiers are used in signal and image processing to estimate the original probability distribution of variables under measurement error~\cite{hohage2024mollifier}. However, to the best of the authors' knowledge, mollifiers have not been applied to path generation problems, despite their natural fit for robotics applications. 

Mollification offers a principled approach to path generation. For example, given a finite collection of ordered waypoints, it produces a smooth curve that approximates the piecewise-linear route to an arbitrary degree of Euclidean closeness. For this specific generated path, our method is well-suited for unicycle-like robots that travel at constant speed with heading-rate constraints; i.e., we can guarantee a maximum curvature for the generated path at the mild cost of slightly deviating from these connecting segments—an unavoidable trade-off, since the union of segments produces a non-differentiable path. Our approach is particularly appealing because it is computationally inexpensive as it can be run on small microcontrollers, and the resulting path, in its parametric representation, can be followed using path-following techniques based on modern guiding vector fields~\cite{weijiaarticlegvf}. We also show how our method can generate 3D paths and guarantee several properties for generic sets of input points, including convexity preservation, enclosure of the generated path, and bounded maximum length when the input is also closed.

The article is organized as follows, Section \ref{sec: problem} introduces the mollifier functions and the requirements for the path generation problem. Section \ref{sec: res} explores which geometrical and analytical aspects of the original function 
are affected by mollification. Convexity, concavity and quasiconvexity are treated, as well as the length of
the path and the curvature. Section \ref{sec: exp} validates the theoretical findings numerically and experimentally. We end the article with some conclusions in Section \ref{sec: con}.

\section{Notation}
\label{sec: not}

In this paper we use Lebesgue integration rather than Riemann integration to 
overcome some limitations with sets of measure zero and interchanging integral 
and limits. Therefore, we denote the Lebesgue measure in one dimension as $\lambda$, i.e., $\lambda := \lambda_1$, and as $\lambda_n$ the $n$'th dimensional Lebesgue measure, and we consider that it is equipped with the Borel sigma algebra. Consequently, all integrals in this paper must be thought as the Lebesgue integral with 
respect to the Lebesgue measure, even if the notation $\lambda$ is sometimes omitted in the integral.
For $p \in [1,\infty)$, we say that a measurable function $f$ is locally $p$-integrable, denoted $f \in L^p_{\mathrm{loc}}(X)$, if it is $p$-integrable in each compact subset of $X$.
We denote by $\id : X \to X$ the identity function defined as
$\id(x) = x$. Then, the indicator function of a set $A$ is defined as 
\begin{equation*}
    \ind_{A}(x) = \begin{cases}
        1 & x \in A, \\
        0 & x \notin A
    \end{cases}.
\end{equation*}
For any two sets $X,Y$ we denote $C(X,Y)$ as the set of continuous functions
from $X$ to $Y$ and for $n \in \N \cup \{\infty\}$ we denote as $C^n(X,Y)$
the set of $n$-times continuously differentiable functions from $X$ to $Y$. For
a set $A$ we denote its closure as $\overline{A}$ and the support of a
real-valued function is defined as $\supp f := \overline{\{x \in \dom f \mid
f(x) \neq 0\}}$, where $\dom f$ is the domain of $f$.
Finally, a (parametric) path in $\R^n$ is a measurable function $f : X \subset \R \to \R^n$. Writing $f = (f_1,\dots,f_n)$, we call $f_i$ the $i$'th component of the path for $i \in \{1,\dots,n\}$. 

\section{Path generation by mollification}
\label{sec: problem}

\subsection{Path generation requirements}
This paper seeks an alternative approach to interpolation and optimization methods for the generation of paths from high-level inputs—such as waypoint sequences—that is computationally efficient, conceptually simple, and has a transparent physical interpretation. Furthermore, the generated paths must be feasible for mobile robots such as unicycles, avoiding unnecessarily complex trajectories. Technically, we consider the transformation of an arbitrary parametric path $f : \R \to \R^n$ into another parametric path that satisfies the following requirements casted as a formal problem.
\begin{problem}[Path generation problem]\label{prob:RegularizationProblem}
    Let $f : \R \to \R^n$ be a parametric path and $\{\ep_i\}_{i=1}^n$ be a collection of positive real numbers. Find a new path $T_{\ep}(f) : \R \to \R^n$ where $T_{\ep}(f) =: \{T_{\ep_i}(f_i)\}_{i=1}^{n}$ and $T_{\ep_i}$ is a functional that acts on each component of $f$ such that:
    \begin{enumerate}
        \item Each component of the parametric path can be made arbitrary close to the original path, that is, $T_{\ep_i}(f_i) \to f_i$ as $\ep_i \to 0$ in some sense of
        convergence.
        \item It provides enough smoothness, that is, for $p \in \N$ with $p \geq 2$, $T_{\ep_i}(f_i) \in C^p(\R,\R)$ for $i \in \{1,\dots,n\}$.
        \item $T_{\ep_i}(f_i)$ is computationally simple. \hfill $\square$
    \end{enumerate}
\end{problem}
The first requirement ensures that the generated path approximates $f$ arbitrarily well through a single independent tuning parameter per dimension, allowing a sequence of functions to be made arbitrarily close to the—potentially non-differentiable everywhere—original path. The second requirement guarantees feasibility for mobile robots like unicycles with speed constraints; we demonstrate how to bound the curvature when the input path consists of concatenated line segments when $n=2$. Finally, the third requirement enables real-time path generation with low-cost hardware.

\subsection{Mollifiers for the path generation}
\label{subsec: Molli}
The solution to Problem \ref{prob:RegularizationProblem} can be obtained by taking a weighted average of the points along the parametric path $f$ through convolution with a certain type of function known as a \textit{mollifier} \cite{Evans2022-PDE}. Let us recall the convolution operation.
\begin{definition}[Convolution]
\label{def:ConvolutionDefinition}
    Let $f,g \in L^1(\R^n)$. The convolution $f*g : \R^n \to \R$ is defined as
    \begin{equation}
        (f*g)(x) := \int_{\R^n}f(y)g(x-y) \, \mathrm{d}\lambda_n(y). \nonumber
    \end{equation}
    \hfill $\square$
\end{definition}
We recall that the convolution is associative, bilinear and commutative. Let us now introduce the mollifier function.
\begin{definition}[Mollifier]\label{def:MollifiersDefinition}
    Let $\varphi \in C^{\infty}(\R^n,\R)$ and for $\ep >0$ define
    $\varphi_{\ep} := \frac{1}{\ep^n}\varphi \circ
    \frac{\id}{\ep}$. We call $\varphi$ a mollifier if it satisfies:
    \begin{enumerate}
        \item $\supp \varphi$ is compact.
        \item $\int_{\R^n}\varphi \, \mathrm{d}\lambda_n = 1$.
        \item For any bounded $f \in C(\R^n,\R)$, $\lim_{\ep\to 0}
        \int_{\R^n}f(x)\varphi_{\ep}(x) \, \mathrm{d}\lambda_n(x) = f(0)$. \hfill $\square$
    \end{enumerate}
\end{definition}
Let us present one of the most popular mollifiers since it will be used extensively in this paper.
\begin{example}\label{example:OurMollifier}
    Let $\varphi : \R \to [0,\infty)$ be the function
    \begin{equation}\label{eq:OurMollifier}
        \varphi(x) = \begin{cases}
            c_1\exp\left(\frac{-1}{1-x^2}\right), & |x| < 1, \\
            0, &|x| \geq 1
        \end{cases},
    \end{equation}
    where $c_1 > 0$ is a normalization constant that ensures $\int_{\R}\varphi \, \mathrm{d}\lambda = 1$. Clearly $\supp \varphi = \overline{(-1,1)} = [-1,1]$ and
    $\supp \varphi_{\ep} = [-\ep,\ep]$. Moreover, with a change
    of variables it can be seen that $\int_{\R}\varphi_{\ep} \, \mathrm{d}\lambda = 1$, 
    and it can also be shown using the Lebesgue Dominated Convergence Theorem
    that as $\ep \to 0$ the third property in Definition \ref{def:MollifiersDefinition} holds. Figure 
    \ref{fig:OurMollifier} represents the function
    $\varphi_{\ep}$ for different values of $\ep > 0$. \hfill $\square$
\end{example}
For several results in this paper, we will require the following assumption.
\begin{assumption}\label{ass: mol}
The mollifier $\varphi$ is nonnegative and its support is the symmetric set around the origin $[-1,1]$.
\end{assumption}
Note that this assumption is made for the sake of convenience. If the set is not symmetric around the origin, many of the presented results will be \emph{displaced} but still apply. Also note that if the support of $\varphi$ is $[-a,a]$, with $0 < a \neq 1$, we can always rescale it to be $[-1,1]$ via the parameter $\ep$ as in Definition \ref{def:MollifiersDefinition}.

\begin{figure}
\centering
\includegraphics[width=\linewidth]{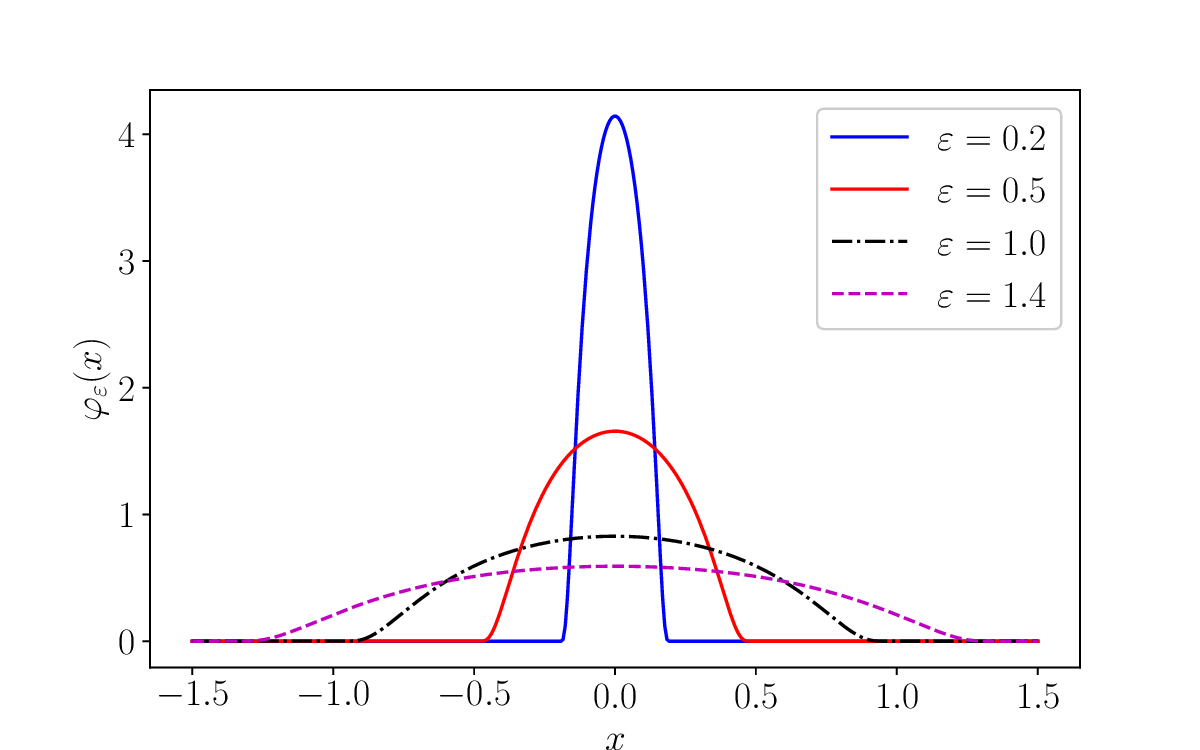}
\caption{Representation of $\varphi_{\ep}$ for different values of $\ep > 0$ where the mollifier $\varphi$ is defined in Example \ref{example:OurMollifier}.}
\label{fig:OurMollifier}
\end{figure}

Now, let us consider a set of mollifiers $\{\varphi_{\ep_i} : \R \to \R\}_{i=1}^{n}$, and define the generated path $F = (F_1,\dots,F_n)$ with $F_i := f_i * \varphi_{\ep_i}$ for all $i \in \{1,\dots,n\}$, and let us denote by $\varphi_{\ep}^{(n)}$ the $n$-th derivative of the function $\varphi_{\ep}$.
\begin{theorem}[{\cite[Appendix C, Theorem
7]{Evans2022-PDE}}]\label{thm:PropertiesOfMollifying}
Let $\ep > 0$, $f \in L_{loc}^p(\R)$ with $p \in [1,\infty]$, and $\varphi$ be a mollifier; then the following three statements hold:
\begin{enumerate}
\item $F_{\ep} \in C^{\infty}(\R,\R)$ and for any $n \in \N$ we have that $F_{\ep}^{(n)} = (\varphi_{\ep} * f)^{(n)} = (\varphi_{\ep})^{(n)} * f$.
\item $\varphi_{\ep} * f \to f$ pointwise almost everywhere as $\ep \to 0$.
\item If $p < \infty$ then $\lim_{\ep\to0}||\varphi_{\ep} * f - f||_{p} = 0$. \hfill $\square$
\end{enumerate}
\end{theorem}
While it may seem otherwise, Theorem \ref{thm:PropertiesOfMollifying} does not solve Problem \ref{prob:RegularizationProblem} \emph{entirely}. Indeed, the second requirement is satisfied straightforwardly. Regarding the computational efficiency from the third requirement, the computation of $F$ is also straightforward. Indeed, because of the compact and symmetric support of the mollifier due to Assumption \ref{ass: mol}, we have that $F_{\ep}(x) = \int_{[-\ep,\ep]}f(x-t)\varphi_{\ep}(t) \, \mathrm{d}\lambda(t)$ is an inexpensive numerical operation, and note the compact integration interval. Furthermore, note that $F_i^{(n)} := (f_i * \varphi_{\ep_i})^{(n)} = f_i * \varphi_{\ep_i}^{(n)}$, and also note that computing $F_i^{(n)}(x)$ does not require the existence of $f_i^{(n)}(x)$. 

However, regarding the first requirement of Problem \ref{prob:RegularizationProblem}, Theorem \ref{thm:PropertiesOfMollifying} only gives us pointwise convergence and $L^p$ convergence. For path following or trajectory tracking algorithms we want to have a stronger notion of convergence for Problem \ref{prob:RegularizationProblem}, i.e., if $f \in L^p_{loc}(\R)$ then for any mollifier $\varphi$, we want $\varphi_{\ep} * f \to f$ as $\ep \to 0$ uniformly. This is true if we require $f$ to be uniformly continuous; nonetheless, if $f$ is just continuous we have uniform convergence on compact subsets of $\R$. Let us finish by showing why mollifying a uniformly continuous $f$ solves the Path Generation Problem \ref{prob:RegularizationProblem} with a stronger notion of convergence.

\begin{theorem}[Uniform convergence]\label{thm:UniformConvergence}
    Let $f \in L^p_{loc}(\R)$ and let $\varphi : \R \to \R$ be a non negative
    mollifier. The following statements hold:
    \begin{enumerate}
        \item If $f$ is uniformly continuous then $F_{\ep} \to f$ as $\ep
        \to 0$ uniformly.
        \item If $f$ is Lipschitz continuous then $F_{\ep} \to f$ as $\ep \to 0$ uniformly and $F_{\ep}$ is Lipschitz continuous for any $\ep > 0$.
        \item If $f$ is continuous then $F_{\ep} \to f$ as $\ep \to 0$ uniformly on compact subsets of $\R$.
    \end{enumerate}
\end{theorem}

\begin{proof}

For the first statement, suppose $f$ is uniformly continuous. Fix $x \in \R$ and $\eta > 0$. We know there exists a $\delta = \delta(\eta) > 0$ such that $|f(a)-f(b)| < \eta$ whenever $|a-b| < \delta$. Choose $\ep \in (0,\delta)$. Recall that $\int_{\R} \varphi = 1$, thus $f(x) = \int_{\R}f(x)\varphi(t) \, \mathrm{d}\lambda(t)$. Therefore
\begin{align*}
|F_{\ep}(x) -f(x)| &= \left|\int_{(-\ep,\ep)}(f(x-t)-f(x))\varphi_{\ep}(t) \, \mathrm{d}\lambda(t)\right| \\
&\leq \int_{(-\ep,\ep)}|f(x-t)-f(x)|\varphi_{\ep}(t) \, \mathrm{d}\lambda(t) < \eta, 
\end{align*}
because $|x-t-x| = |t| \leq \ep < \delta$. Since $x$ was arbitrary the statement follows.

For the second statement, since every Lipschitz continuous functions is uniformly continuous, the uniform convergence claim follows from the previous  paragraph. Moreover, take $x, y \in \R$, $\ep \in (0,\infty)$ and suppose that $f$ is Lipschitz with Lipschitz constant $K > 0$.
\begin{align*}
|F_{\ep}(x) - F_{\ep}(y)| &\leq \int_{\R}|f(x-t)-f(y-t)|\varphi_{\ep}(t)\, \mathrm{d}\lambda(t) \\
&\leq K|x-y|\int_{\R}\varphi_{\ep}(t) \, \mathrm{d}\lambda(t) = K|x-y|,
\end{align*}        
which proves that $F_{\ep}$ is Lipschitz.

Finally, for the third statement, suppose now $f$ is continuous and take any compact set $K \subset \R$.  Then $f$ is uniformly continuous on $K$. Thus, the same arguments as above can be followed noting that in this case $\delta$ depends on the compact set $K$.
\end{proof}

\section{Key properties of the generated path}
\label{sec: res}

While Theorems \ref{thm:PropertiesOfMollifying} and \ref{thm:UniformConvergence} solve Problem \ref{prob:RegularizationProblem}, they do not provide further details about the properties of the resulting path. In this section, we characterize key properties of the generated path based on the input path. Specifically, we address questions such as: under what conditions does the output path preserve (local/quasi) convexity of the input? Does the output maintain monotonicity or other qualitative properties of the input, such as for step functions? How is the output path positioned relative to the input? How is the output path enclosed, and what is its length when the input is a closed path? We defer the analysis of the output path's curvature to Section \ref{sec: curvature}, where we provide a detailed curvature analysis for the case when the input is a sequence of 2D or 3D waypoints. This curvature analysis serves as a systematic methodology that can be applied to other types of input paths.

For conciseness, we restrict our attention to mollifiers defined on the real line, since our analysis is carried out component-wise along trajectories. Nevertheless, most results can be generalized to functions from $\R^n$ to $\R$, which may be useful when the desired path in $\R^n$ is encoded as the intersection of $n-1$ hypersurfaces parametrized by functions from $\R^n$ to $\R$. In such higher-dimensional cases, one would consider the standard Euclidean norm in $\R^n$, the topology whose basis consists of open balls $B(x,\ep) = \{y \in \R^n \mid ||y-x|| <\ep\}$, and a mollifier whose support is the closed ball $\overline{B}(0,1) = \{x \in \R^n \mid ||x|| \leq 1\}$. Indeed, this is precisely the case for the extension to $\R^n$ of the mollifier presented in Example \ref{example:OurMollifier}.

\subsection{Convexity properties}
Convexity and local convexity are properties of great interest in the study of trajectory shapes. For example, if a vehicle attempts to follow a continuous but non-differentiable trajectory that resembles an inverted tent, such as the function $x \in \R \mapsto |x|$, will the mollified trajectory preserve this inverted tent shape? What if the property holds only locally? Intuitively, since mollification is a weighted average of the original function, the answer to the first question is affirmative. However, the answer to the second question depends on the parameter value $\varepsilon$. If the parameter is sufficiently large, the ``average'' of the function over the locally convex region may become negligible. We now present several propositions and counterexamples addressing these questions.

\begin{proposition}[Convexity and mollification]\label{prop:ConvexityIsPreserved}
    Let $f \in L^1_{loc}(\R)$ be convex and $\varphi$ a nonnegative
    mollifier as in Definition \ref{def:MollifiersDefinition}; then $F_{\ep} :=
    \varphi_{\ep} * f$ is convex for any $\ep > 0$.
\end{proposition}
\begin{proof}
    Let $x,y \in \R$ and $\gamma \in [0,1]$. Since $\varphi_{\ep} \geq 0$
    we have that
    \begin{align*}
        F_{\ep}(\gamma x + &(1-\gamma)y) = \int_{\R}f(\gamma x +
        (1-\gamma)y - t)\varphi_{\ep}(t) \, \mathrm{d}t \\
        &= \int_{\R}f(\gamma(x-t) + (1-\gamma)(y-t))\varphi_{\ep}(t) \, \mathrm{d}t \\
        &\leq \int_{\R}\left[\gamma f(x-t) +
        (1-\gamma)f(y-t)\right]\varphi_{\ep}(t) \, \mathrm{d}t \\
        &=\gamma F_{\ep}(x) + (1-\gamma)F_{\ep}(y).
    \end{align*}
\end{proof}

This property allows us to predict the shape of the mollified trajectory in advance. For example, if the trajectory to be followed resembles an inverted tent, the mollified trajectory $F_{\ep}$ will also resemble an inverted tent for any $\ep>0$. The question is whether local convexity is always preserved. This is false; local convexity is preserved only for sufficiently small $\ep>0$, where the bound on $\ep$ depends on the neighborhood in which the function is convex. We now present a proposition and a counterexample.

\begin{proposition}[Local convexity and mollification]\label{prop:LocalConvexity}
    Let $f \in L^{1}_{loc}(\R)$ be a function that is convex in some set $(a,b)
    \subseteq \R$ with $-\infty \leq a < b \leq \infty$. Let $\varphi$ satisfy Assumption \ref{ass: mol}. Then, for each $x,y \in (a,b)$
    with $x < y$ there exists a $\delta = \delta(x,y)> 0$ such that for all $\ep
    \in(0,\delta)$ the function $F_{\ep} := f * \varphi_{\ep}$ is
    convex in the set $(x,y)$.
\end{proposition}
\begin{proof}
    Let $x,y \in (a,b)$.
    Since $(a,b)$ is open there exists a real number $\delta > 0$ such that
    $(x-\delta,y+\delta) \subset (a,b)$. 
    The sets $V = (x-\delta,y+\delta)\subset (a,b)$ and $(x,y)\subset(a,b)$ are
    clearly open and convex. Choose $\xi,\zeta \in (x,y)$ and $\gamma \in
    [0,1]$. Then we have that $\gamma \xi + (1-\gamma)\zeta \in (x,y)$. Let $\ep \in (0,\delta)$, and we know that
    \begin{align*}
        F_{\ep}(\gamma &\xi + (1-\gamma)\zeta) = \int_{\R}f(\gamma \xi +
        (1-\gamma)\zeta -t)\varphi_{\ep}(t) \, \mathrm{d}\lambda(t).
    \end{align*}
    Note that by the selection of $\ep$ and $\delta$ we have that $t \in
    (-\ep,\ep) \subset (-\delta, \delta)$. Thus, for any $t \in
    (-\ep,\ep)$,
    $\xi-t \in
    V$ and $\zeta-t \in V$. Since $V$
    is convex we have that $\gamma (\xi-t) + (1-\gamma)(\zeta-t) \in V$ for all
    $t \in (-\ep,\ep)$ and $\gamma \in [0,1]$. Given that $f$ is convex
    in $(a,b)$ it is also convex in $V\subset (a,b)$, and noting
    $\varphi_{\ep} \geq 0$ we can follow the steps of the proof of proposition
    \ref{prop:ConvexityIsPreserved} to reach
    \begin{align*}
        F_{\ep}(\gamma \xi + (1-\gamma)\zeta)
        &\leq\gamma F_{\ep}(\xi) + (1-\gamma)F_{\ep}(\zeta).
    \end{align*}
    Because $\xi,\zeta \in V$ and $\gamma \in [0,1]$ were arbitrary the
    proposition follows.
\end{proof}

\begin{example}\label{example: LocalConvexityNotPreserved}
It is natural to ask whether Proposition \ref{prop:LocalConvexity} holds for any $\ep > 0$. That is, if $f$ is convex on an open set, is $F_{\ep}$ also convex on that open set independently of $\ep$ and the choice of non-negative mollifier? This is generally false, as demonstrated by the following counterexample. Consider the continuous function
\begin{equation}
f(x) = \begin{cases}
        0, & x < 0 \\
        x, & 0 \leq x \leq \frac{1}{2} \\
        1-x, & x > \frac{1}{2}
    \end{cases},
    \label{ex: 2}
\end{equation}
and the mollifier of Example \ref{example:OurMollifier}. The function \eqref{ex: 2} is convex on the open set $(-0.5, 0.5)$. However, for $\ep = 3.2$, the mollified function is not convex and even lies below $f$ at every point in that open set, rather than above, as shown in Figure \ref{fig:CounterExampleLocalConvexity}.
    \begin{figure}[!htb]
        \centering
        \includegraphics[width=\linewidth]{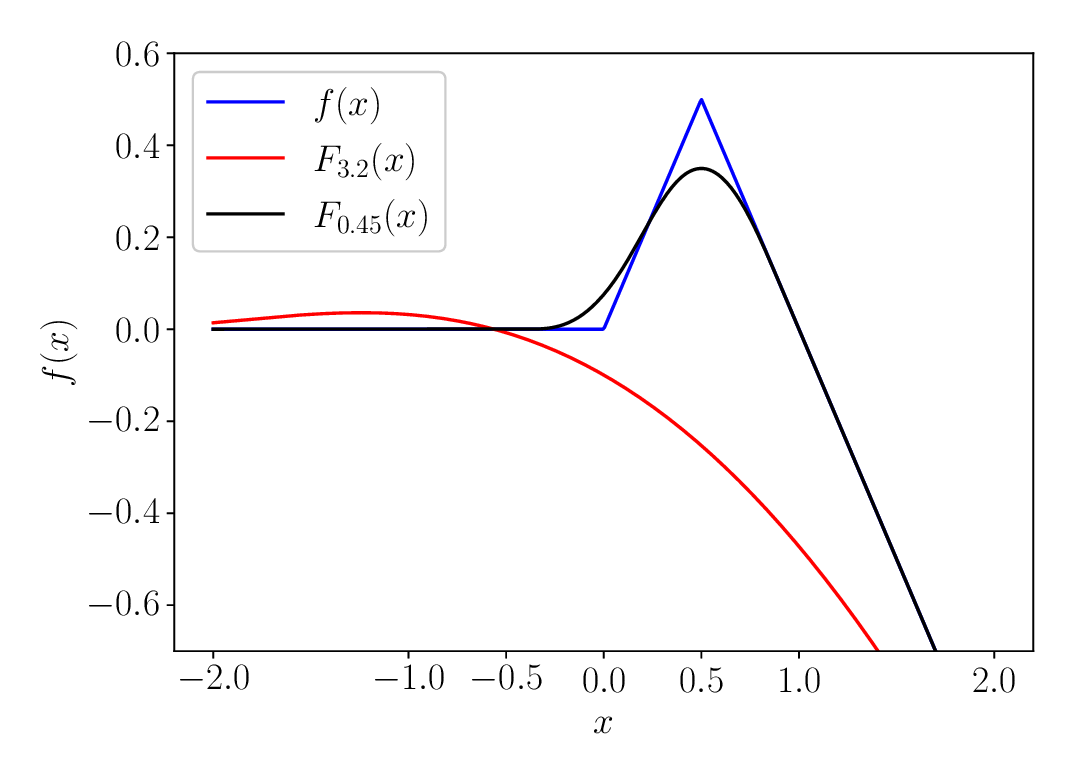}
        \caption{If $f$ is convex in an open neighbourhood there can
            exist a mollifier  $\varphi$ as in example
            \ref{example:OurMollifier} and an $\ep > 0$ such that $F_{\ep} := f
            * \varphi$ is below $f$ in that neighbourhood. Blue line represents
            $f$, while the red one shows $F_{3.2}$ and the black one
            $F_{0.45}$.}
        \label{fig:CounterExampleLocalConvexity}
    \end{figure}
Nevertheless, for $\ep = 0.45$, we have that $F_{\ep}$ is convex on a neighborhood contained in $(-0.5,0.5)$. Thus, we can set $\delta = 0.45$ in Proposition \ref{prop:LocalConvexity}, and for any $\ep < \delta$, $F_{\ep}$ is convex in $(-0.5,0.5)$.
\end{example}
\begin{remark}
Note that Propositions \ref{prop:ConvexityIsPreserved} and \ref{prop:LocalConvexity} also hold when $f$ is concave instead of convex, with concavity replacing convexity throughout.
\end{remark}

Before the following result, we need to prove that affine maps are invariant under mollification.
\begin{proposition}[Affine functions and
    mollification]\label{prop:MollifyIdentityIsIdentity}
    Let $\varphi_{\ep}$ with $\ep > 0$ be a mollifier with symmetric support around the origin, and let $a,b \in \R$; then $\varphi_{\ep} * (a\id+b) =
    a\id + b$.
\end{proposition}
\begin{proof}
    Let $x \in \R$ and $\ep > 0$. Then
    \begin{align*}
        (\varphi_{\ep} * (a\id+b))(x)  &=
        \int_{\R}(a(x-y)+b)\varphi_{\ep}(y) \mathrm{d} \lambda(y)  \\
        &= ax\int_{\R}\varphi_{\ep}(y) \, \mathrm{d}\lambda(y) -
        a\int_{\R}y\varphi_{\ep}(y) \, \mathrm{d}\lambda(y) \\
        & +
        b\int_{\R}\varphi_{\ep}(y)\, \mathrm{d}\lambda(y) \\
        &=ax + b - a\int_{(-\ep,\ep)}y\varphi_{\ep}(y)\, \mathrm{d}\lambda(y).
    \end{align*}
    The result follows noting that $y\in\R\mapsto y \varphi_{\ep}(y)$ is an odd
    function that is integrated over a symmetric interval.
\end{proof}
We also need the Jensen's inequality.
\begin{theorem}[Jensen's inequality]\label{thm:Jensen}
\cite[Theorem 1.6.2]{durrett2019probability}
    Let $\varphi$ be a non-negative measurable function such that
    $\int_{\R}\varphi \mathrm{d}\lambda = 1$, $g$ be any measurable function and $f$ be a convex function such that $\dom f \supset \img g$; then
    \begin{equation*}
        f\left(\int_{\R}g(x)\varphi(x)\mathrm{d}\lambda(x)\right) \leq \int_{\R} (f \circ
        g)(x)\varphi(x)\mathrm{d}\lambda(x).
    \end{equation*}
    \hfill $\square$
\end{theorem}

\begin{figure}
    \centering
    \includegraphics[width=\linewidth]{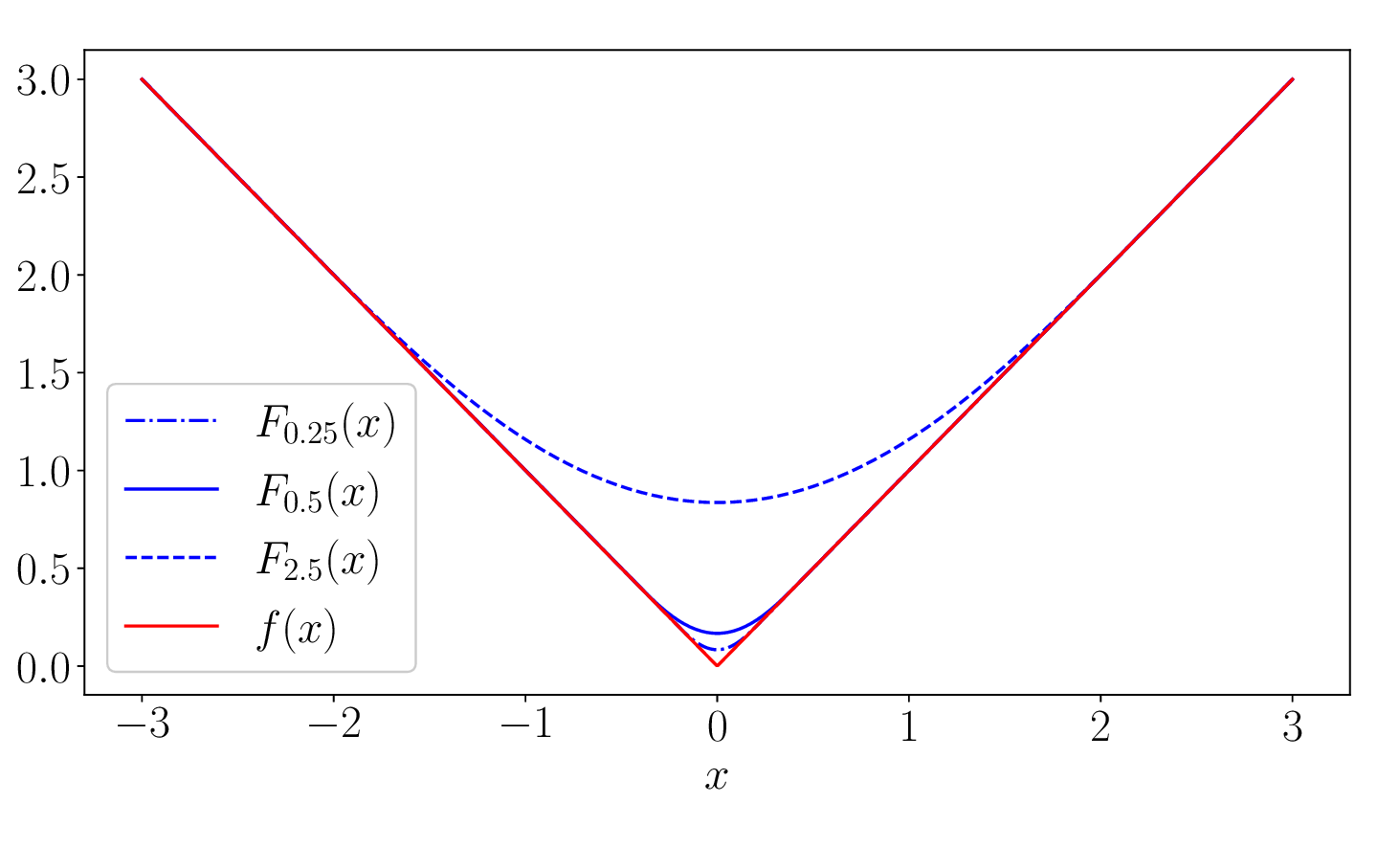}
    \caption{Mollification of the function $f = |\id|$ with the mollifier
        $\varphi_{\ep}$ as in Example \ref{example:OurMollifier} for different
        values of $\ep$. Note that the function $f$ is convex, implying
        $F_{\ep}$ is convex as demonstrated in proposition
        \ref{prop:ConvexityIsPreserved} and it is above the graph as shown in
        Proposition \ref{prop:ConvexityThenMolliAbove}.  }
    \label{fig:MollifOfAbs}
\end{figure}

Now, we are ready to show that if $f$ is convex and $\varphi_{\ep}$ is a non-negative even mollifier, then $\varphi_{\ep} * f \geq f$ pointwise. For example, combined with Proposition \ref{prop:MollifyIdentityIsIdentity}, the following result implies that if our trajectory resembles an inverted tent, then the mollified function will also resemble a smoothed inverted tent but with the mollified function lying \emph{above} the original path; see Figure \ref{fig:MollifOfAbs} for the illustration. 

\begin{proposition}[If $f$ is convex then $F_{\ep}$ is above
    $f$]\label{prop:ConvexityThenMolliAbove}
    Let $\varphi_{\ep}$ with $\ep > 0$ satisfying Assumption \ref{ass: mol}, $f$ be a convex function and define $F_{\ep} := \varphi_{\ep} * f$. Then $F_{\ep} \geq f$.
\end{proposition}
\begin{proof}
    Note that $\varphi_{\ep}$ satisfies the conditions in Jensen's
    inequality and that $f$ is convex. Let $x \in \R$, and note that the
    function $y \in \R \mapsto x-y$ is continuous, and therefore measurable.
    Applying Jensen's inequality
    \begin{align*}
        F_{\ep}(x) = (\varphi_{\ep} * f)(x) &=
        \int_{\R}f(x-y)\varphi_{\ep}(y)\mathrm{d}\lambda(y) \\
        &\geq f\left(\int_{\R}(x-y)\varphi_{\ep}(y)\mathrm{d}\lambda(y)\right) \\
        &= f( (\varphi_{\ep} * \id)(x)).
    \end{align*}
    The result now follows from Proposition \ref{prop:MollifyIdentityIsIdentity}
    since $x$ is arbitrary.
\end{proof}
\begin{remark}
    Note that a similar result also holds if $f$ is concave, where clearly
    $F_{\ep}$ would be below $f$ in that case.
\end{remark}

The following result resembles Proposition \ref{prop:LocalConvexity}.

\begin{proposition}[Local convexity and $F_{\ep}\geq f$]\label{prop:MolliAboveFConvexityLocally}
    Let $\varphi$ be even and satisfying Assumption \ref{ass: mol}, and $f : \R \to \R$ be a 
    function that is convex in $(a,b)$. Then, for each $x,y \in (a,b)$ with $x < y$ there exists a $\delta  =  \delta(x,y)> 0$ such that for all $\ep \in (0,\delta)$ we have that $F_{\ep} = \varphi_{\ep} * f \geq f$ on $(x,y)$.
\end{proposition}
\begin{proof}
    Let $x,y \in (a,b)$.
    There exists a $\delta>0$ such that $(x-\delta,y+\delta) \subset (a,b)$.
    Choose $\ep \in (0,\delta)$ and take $\xi \in (x,y) \subset (a,b)$.
    Then $(\xi-\ep,\xi+\ep) \subset (x-\ep,y+\ep) \subset
    (x-\delta,y+\delta)\subset (a,b)$.  Since $f$ is convex in
    $(\xi-\ep,\xi+\ep)$ we can apply Jensen's inequality leading to
    \begin{align*}
        F_{\ep}(\xi) &= (\varphi_{\ep} * f)(\xi) 
        =
        \int_{(-\ep,\ep)}f(\xi-t)\varphi_{\ep}(t)\mathrm{d}\lambda(t) \\
        &\geq
        f\left(\int_{(-\ep,\ep)}(\xi-t)\varphi_{\ep}(t)\mathrm{d}\lambda(t)\right) \\
        &= f( (\varphi_{\ep} * \id)(\xi))
        = f(\xi),
    \end{align*}
    where the last equality comes from Proposition
    \ref{prop:MollifyIdentityIsIdentity}. Since $\xi \in (x,y)$ was arbitrary
    the result follows.
\end{proof}

\begin{remark}
Note that as it can be seen in Example \ref{example: LocalConvexityNotPreserved}, Proposition
    \ref{prop:MolliAboveFConvexityLocally} does not hold for any $\ep > 0$. Indeed, also note that if the function is locally concave then we reach a similar result where the original function will be above the mollified function.
\end{remark}

\subsection{Quasiconvexity under mollification}
Any convex function is quasiconvex, but the converse does not hold in general. Therefore, quasiconvexity is a weaker condition than convexity, and it is important to study how it is preserved under mollification. From an engineering standpoint, this is particularly useful for estimating the shape of the mollified curve.
\begin{definition}[Quasiconvex function]\label{def:Quasiconvex}
    Let $S$ be a non empty convex set and $f : S \to \R$ be a real-valued
    function. We say that $f$ is quasiconvex if for all $\alpha \in \R$ the set
    \begin{equation*}
        S_{\alpha} := \{x \in S \mid f(x) \leq \alpha\}
    \end{equation*}
    is  convex.
    An equivalent definition is that $f$ is quasiconvex if for all $x,y\in S$
    and all $\gamma \in [0,1]$ we have
    \begin{equation*}
        f(\gamma x + (1-\gamma)y) \leq \max \{f(x),f(y)\}. \; \square
    \end{equation*}
    
\end{definition}

We need the following lemma before proving that quasiconvexity is preserved
under mollification.
\begin{lemma}\label{lemma:QuasiconvexFunctionMeasurable}
    Let $f : \R \to \R$ be a quasiconvex function, $a,b \in \R$, and define
    $g := a \id + b$; then, the following statements hold:
    \begin{enumerate}
        \item $f$ is measurable, and
        \item $f \circ g$ is quasiconvex.
    \end{enumerate}
\end{lemma}
\begin{proof}
First we prove that $f$ is measurable. Note that for any $\alpha \in \R$ the
    set $S_{\alpha} = \{x \in \R \mid f(x) \leq \alpha\}$ is a convex set,
    therefore path connected, hence connected. Since a set in $\R$ is connected
    if and only if it is an interval, then $S_{\alpha}$ is an interval. Thus, it is measurable because any interval is a Borel set.

    We now prove that $f \circ g$ is quasiconvex.    Let $\alpha \in \R$ and
    define the set
    \begin{equation*}
        S_{\alpha} := \{x \in \R \mid (f \circ g)(x) \leq \alpha\} = \{x \in \R
        \mid f(ax  + b) \leq \alpha\}.
    \end{equation*}
    If $S_{\alpha}$ is empty then it is convex by definition, and if
    $S_{\alpha}$ is a singleton is also convex. Therefore, suppose that
    $S_{\alpha}$ consists of at least two elements. Choose $x,y \in S_{\alpha}$
    and $\gamma \in [0,1]$. Then
    \begin{align*}
        f(a(\gamma x + (1-\gamma)y) + b) &= f(\gamma(ax+b) + (1-\gamma)(ay+b)) \\
        &\leq \max \{f(ax+b), f(ay+b)\} \\
        &\leq \max \{\alpha,\alpha\} \\
        &= \alpha.
    \end{align*}
    That is, $\gamma  x + (1-\gamma)y \in S_{\alpha}$. Therefore $S_{\alpha}$ is
    a convex set, proving $f \circ g$ is quasiconvex.
\end{proof}

Having presented the main properties of quasiconvex functions, let us consider a representative quasiconvex case, the monotonic function. Interesting enough, the mollified function of a monotonic path, e.g., a staircase-like sequence of steps, will also be monotonic.
\begin{proposition}[Monotonicity and mollification]\label{prop:Monotonicity}
    Let $f \in L^1_{loc}(\R)$ be monotone increasing (resp. decreasing); then
    for any nonnegative mollifier $\varphi$ and $\ep > 0$ the function
    $F_{\ep} := (\varphi_{\ep} * f)$ is monotone increasing (resp.
    decreasing).  
\end{proposition}
\begin{proof}
    Suppose $f$ is monotone increasing. Let $x,y \in \R$ with $x > y$. Then for
    any $t \in \R$ we have $x-t > y-t$, thus $f(x-t) \geq f(y-t)$.
    \begin{align*}
        F_{\ep}(x)-F_{\ep}(y) &=
        \int_{\R}[f(x-t)-f(y-t)]\varphi_{\ep}(t)\mathrm{d}\lambda(t) \geq 0,
    \end{align*}
    since $\varphi_{\ep}$ is positive. The proof is identical for monotone
    decreasing functions.
\end{proof}

\begin{figure}[ht]
    \centering
    \includegraphics[width=\linewidth]{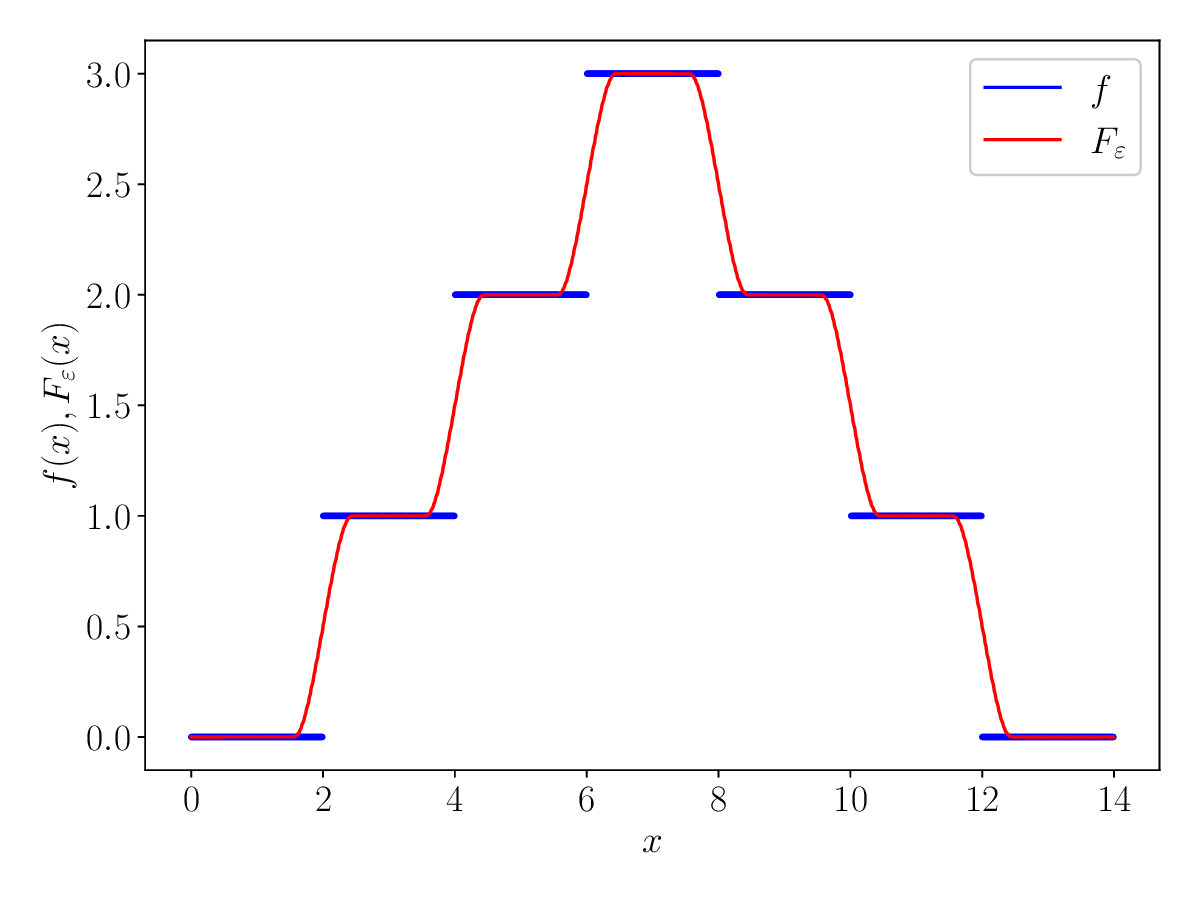}
    \caption{Mollification of stair-step function created using 
    shifted Heaviside step functions. The blue solid line represents the 
    (discontinuous) stair-step function $f$, while the red solid line represents its 
    mollification using the mollifier \eqref{eq:OurMollifier} and 
    $\ep = 0.5$. Due to Proposition \ref{prop:Monotonicity}, $F_{\ep}$ cannot 
    present overshoots or oscillations.}
    \label{fig:MollificationOfStepStairFunction}
\end{figure}
Suppose the desired trajectory is the Heaviside step function $h :\R \to \R$
defined by $h(t) = \ind_{(0,\infty)}(t)$. By this result, its mollification is
also monotonically increasing, so the mollified trajectory cannot exhibit
overshoots or oscillations, see Figure
\ref{fig:MollificationOfStepStairFunction}; this contrasts with the Gibbs
phenomenon. We finally
show that quasiconvexity is preserved under mollification with a nonnegative
mollifier.
\begin{theorem}\label{thm:QuasiconvexityIsPreserved}
    Suppose $f : \R\to \R$ is a quasiconvex function and let $\varphi$ be a mollifier satisfying Assumption \ref{ass: mol}; then, for all $\ep >
    0$ the function $F_{\ep} := f * \varphi_{\ep}$ is quasiconvex.
\end{theorem}

\begin{proof}
    We are going to proceed by contradiction. Let $f$ be quasiconvex, and suppose that there exists an $\ep > 0$
    such that $F_{\ep}$ is not quasiconvex.
    This implies there exist $x,y \in \R$ and $\gamma \in [0,1]$ such that
    \begin{equation*}
        F(\gamma x + (1-\gamma)y) > \max\{F(x),F(y)\}.
    \end{equation*}
    In particular
    \begin{align*}
        &F_{\ep}(\gamma x + (1-\gamma)y) - F_{\ep}(x) \\
        &=
        \int_{(-\ep,\ep)}\left[f(\gamma x +
        (1-\gamma)y-t)-f(x-t)\right]\varphi_{\ep}(t)\mathrm{d}\lambda(t) > 0.
    \end{align*}
    First, we claim that the set $V_1 := \{t \in (-\ep,\ep) \mid f(\gamma x +
    (1-\gamma)y-t)-f(x-t) > 0\}$ is measurable and it has positive measure.
    By Lemma \ref{lemma:QuasiconvexFunctionMeasurable} we know that the composition of a quasiconvex function with an affine mapping is quasiconvex, hence
    measurable. Since the linear combination of two measurable function results
    in a measurable function, the set $V_1$ is measurable. Thus, if
    $\lambda(V_1)$ were zero, then it would be negligible in integration, which
    would imply that
    \begin{align*}
        0&< \int_{(-\ep,\ep)}\left[f(\gamma x +
        (1-\gamma)y-t)-f(x-t)\right]\varphi_{\ep}(t)\mathrm{d}\lambda(t) \\
        &=\int_{\left[(-\ep,\ep)\setminus V_1\right] \cup V_1}\left[f(\gamma x +
        (1-\gamma)y-t)-f(x-t)\right]\varphi_{\ep}(t)\mathrm{d}\lambda(t)  \\
        &=\int_{(-\ep,\ep)\setminus V_1}\left[f(\gamma x +
        (1-\gamma)y-t)-f(x-t)\right]\varphi_{\ep}(t)\mathrm{d}\lambda(t),
    \end{align*}
    but $(-\ep,\ep)\setminus V_1 = \{t \in (-\ep,\ep) \mid t \notin V_1\} = \{t
    \in (-\ep,\ep) \mid f(\gamma x + (1-\gamma)y-t)-f(x-t) \leq 0\}$ and
    $\varphi$ is nonnegative, thus
    \begin{equation*}
        \int_{(-\ep,\ep)\setminus V_1}\left[f(\gamma x +
        (1-\gamma)y-t)-f(x-t)\right]\varphi_{\ep}(t)\mathrm{d}\lambda(t) \leq 0,
    \end{equation*}
    which leads to the first contradiction. Therefore $\lambda(V_1) > 0 \Longrightarrow V_1
    \neq \emptyset$. Moreover, since $f$ is quasiconvex we have that for any $t
    \in V_1$
    \begin{align*}
        f(x-t) &< f(\gamma x + (1-\gamma)y -t)  \\
        &=
        f(\gamma(x-t)+(1-\gamma)(y-t)) \\
        &\leq \max\{f(x-t),f(y-t)\}.
    \end{align*}
    Now let us consider that
    \begin{equation}\label{eq:ConjectureProofEq1}
        f(y-t) > f(x-t), \quad \forall t \in V_1.
    \end{equation}
    because if not, we would reach a contradiction; therefore, proving that our assumption
    about $F_{\ep}$ is false, thus proving the theorem.

    The same procedure can be done considering the other point, $y$, and
    $F_{\ep}(\gamma x + (1-\gamma)y) - F_{\ep}(y) > 0$, i.e., there exists a
    set of positive measure $V_2 \subset (-\ep,\ep)$ such that
    \begin{equation*}
        f(\gamma x + (1-\gamma)y -t) - f(y-t) > 0, \quad \forall t \in V_2,
    \end{equation*}
    and since $f$ is quasiconvex, with the same arguments as above
    \begin{equation}\label{eq:ConjectureProofEq2}
        f(x-t) > f(y-t), \quad \forall t \in V_2.
    \end{equation}
    Suppose $V_1 \cap V_2 \neq \emptyset$. This implies that there exists a $t
    \in V_1\cap V_2$ such that $f(y-t) < f(x-t) < f(y-t)$ which is again a
    contradiction, thus proving the theorem. So we just need to prove that we
    reach a contradiction in the case $V_1 \cap V_2 = \emptyset$. Take $t_1 \in
    V_1$ and $t_2 \in V_2$ and any $\beta \in [0,1]$, then
    \begin{align*}
        f(y-\beta t_1 - (1-\beta)t_2) &= f(\beta(y-t_1)+(1-\beta)(y-t_2)) \\
        &\leq
        \max \{f(y-t_1), f(y-t_2)\}
    \end{align*}
    However, $f(y-t_2) < f(x-t_2)$ because of \eqref{eq:ConjectureProofEq2},
    thus
    \begin{align*}
        f(y-\beta t_1 - (1-\beta)t_2)< \max\{f(y-t_1), f(x-t_2)\}, \\ 
        \quad \forall
        \beta \in [0,1].
    \end{align*}
    Using the same approach and considering 
    \eqref{eq:ConjectureProofEq1} we have that
    \begin{align*}
        f(x-\beta t_1 - (1-\beta)t_2) < \max\{f(y-t_1), f(x-t_2)\}, \\ \quad
        \forall \beta \in [0,1].
    \end{align*}
    Since this is independent of the value of $\beta$,  and in particular for
    $\beta = 1$ and $\beta = 0$ we have that
    \begin{align*}
        f(y-t_1) < \max\{f(y-t_1),f(x-t_2)\} \\
        f(x-t_2) < \max\{f(y-t_1),f(x-t_2)\},
    \end{align*}
    but this leads to a contradiction. Therefore, the assumption that there
    exists an $\ep > 0$ for which $F_{\ep}$ is not quasiconvex is false. That
    is, $F_{\ep}$ is quasiconvex for any $\ep > 0$.
\end{proof}

\begin{remark}
\label{rem: Q}
    The converse is not in general true. That is, having $F_{\ep}$ quasiconvex
    for some $\ep > 0$ does not imply that $f$ is quasiconvex. For example, 
    consider the following measurable function that is not quasiconvex
    \begin{equation*}
       f(x) = \begin{cases}
            1, & x\in \R \setminus \Q \\
            2, & x \in \Q
        \end{cases},
    \end{equation*}
    because the set $S_{1.5} = \{x \in \R \mid
    f(x) \leq 1.5\} = \R \setminus \Q$ is disconnected, hence it is not convex.
    However, since $\lambda(\Q)= 0$, we have that for any $x \in \R$
    \begin{align*}
        F_{\ep}(x) &= \int_{\R}\varphi(x-t)f(t)\mathrm{d}\lambda(t) 
        = \int_{\R \setminus
            \Q}\varphi(x-t)\mathrm{d}\lambda(t) \\
            &=
        \int_{(\R\setminus \Q) \cup \Q} \varphi(x-t) \mathrm{d}\lambda(t) = \int_{\R}
        \varphi \mathrm{d}\lambda = 1,
    \end{align*}
    which is convex; thus quasiconvex too.  Moreover, it is clear that if $F_{\ep}$ is quasiconvex for all
    $\ep > 0$ then $f$ is quasiconvex since the pointwise limit of a
    family of quasiconvex functions can be shown to be quasiconvex function as well.
\end{remark}

\subsection{Enclosure and length of paths}
So far, we have worked with real-valued functions because we represent desired paths or trajectories as parametric functions, i.e., functions $f : \R \to \R^n$. 
Nonetheless, we still want to characterize (in advance) how mollification affects the complete function $f$, that is, treat it as a whole. We now address the following question: given the original trajectory, does there exist a subset $U$ of $\R^n$ such that the mollified trajectory is contained in $U$ for any value of its parameter? The answer is affirmative, with $U$ being the convex hull of $f(\R)$.
\begin{definition}
    Let $A \subset \R^n$ be a set. Its convex hull, denoted as $\co(A)$ is defined
    as the smallest convex set that contains $A$, that is, $A \subset \co(A)$.  \hfill $\square$
\end{definition}

We first present a result when $\dom f = \R$.
\begin{theorem}\label{thm:ConvexHull}
    Let $f : \R \to \R^n$ be a measurable function and $\varphi$ be a nonnegative
    mollifier. Define for $t \in \R$ and $\ep > 0$
    \begin{equation*}
        F_{\ep}(t) := (f * \varphi_{\ep})(t) = \left((f_1 * \varphi_{\ep})(t),
        \dots, (f_{n}*\varphi_{\ep})(t)\right).
    \end{equation*}
    Then, given $\ep > 0$, we have that
    \begin{equation*}
        \left\{F_{\ep}(t) \mid t \in \R\right\} \subset \co\left\{f(t) \mid t \in \R\right\}.
    \end{equation*}
\end{theorem}
\begin{proof}
    Let $U = \co\{f(t) \mid t \in \R\}$. Define the extended real valued 
    function $I_{U} : \R^n \to \R \cup \{-\infty,\infty\}$ as
    \begin{equation*}
        I_{U}(x) = \begin{cases}
            +\infty, & x \notin U  \\
            0, & x \in U
        \end{cases}.
    \end{equation*}
    The function $I_{U}$ is clearly convex. Fix $ t \in \dom f$. Noting that
    $\int_{\R}\varphi_{\ep}\mathrm{d}\lambda = 1$ and $\varphi_{\ep} \geq 0$, we can
    apply Jensens' Inequality in higher dimensions to get
    \begin{align*}
        0 \leq I_U(F_{\ep}(t)) &=
        I_{U}\left(\int_{\supp
        \varphi_{\ep}}f(t-s)\varphi_{\ep}(s)\mathrm{d}\lambda(s)\right) \\
        &\leq
        \int_{\supp \varphi_{\ep}}I_{U}(f(t-s))\varphi_{\ep}(s)\mathrm{d}\lambda(s).
    \end{align*}
    However, note that $I_{U}(f(t-s)) = 0$ for any $t-s\in \dom f$, and since 
    $\dom f = \R$ then 
    \begin{equation*}
        0 \leq I_U(F_{\ep}(t)) \leq 
        \int_{\supp \varphi_{\ep}}I_{U}(f(t-s))\varphi_{\ep}(s)\mathrm{d}\lambda(s) = 0,
    \end{equation*}
    i.e., $I_U(F_{\ep}(t)) = 0$ so $F_{\ep}(t) \in U$. Since $t$ and $\ep$ 
    are arbitrary, the claim follows.
\end{proof}
\begin{remark}\label{rmk:ContinuousExtension}
    It is common that the path is a continuous
    function defined in a compact subset of $\R$, i.e., $f : [a,b] \to \R^n$
    with $-\infty < a < b < \infty$. In such a case, we can extend the function $f$ to $\R$ as follows to get a new continuous function
\begin{equation*}
    \bar f(t) = \begin{cases}
        f(a), & -\infty < t \leq a \\
        f(x), & a \leq t \leq  b \\
        f(b), & b \leq t < \infty
    \end{cases}.
\end{equation*}
Note that $\bar f([a,b]) = f([a,b])$ and $\bar f((-\infty,a]\cup[b,\infty))
= \{f(a),f(b)\}$, so $\bar{f}(\R) = f([a,b])$ and then $\co \bar{f}(\R) =
\co f([a,b])$. Then we can use as our path $\bar{f}$ instead of $f$,
obtaining the result of the previous theorem, and later restricting 
the domain of the mollified function to $[a,b]$ again, i.e.,
we let $\bar F_{\ep} = \bar{f} * \varphi_{\ep}$ and use the mollified
curve $F_{\ep} = \bar{F}_{\ep}|_{[a,b]}$, thus
\begin{equation*}
    F_{\ep}([a,b]) \subset \bar{F}(\R) \subset \co \bar f(\R) = \co
    f([a,b]).
\end{equation*}
Clearly for $t \in [a+\ep,b-\ep]$, $F_{\ep}$ coincides with the
mollification of $f$, and in $[a,a+\ep)$ and $(b-\ep,b]$ it
belongs to the convex hull of $f([a,b])$.
\end{remark}

Having characterized the space in which the mollified path is enclosed, we now consider the relationship between the length of the original path and its mollification. First, we introduce the definition of path length for paths that do not need to be differentiable.
\begin{definition}[Length of $f$]
    Let $f : [a,b] \to \R^n$ be a continuous function, and $||\cdot|| : \R^n \to [0,\infty)$ be any norm in $\R^n$. Let a finite set $P = \{x_0,x_1,\dots,x_N\}$, where $a = x_0 < x_1 < \dots < x_N = b$ be a partition of $[a,b]$. Then, the length of $f$ is
    \begin{equation*}
        L(f) := \sup_{P \text{ partition of } [a,b]}\sum_{i=1}^{N}||f(x_i)-f(x_{i-1})||. \, \square
    \end{equation*}
\end{definition}

Note that when working with trajectories with compact domain, we must extend them as done in Remark \ref{rmk:ContinuousExtension}.
\begin{lemma}\label{lem:LengthOfPaths}
    Let $f : [a,b] \to \R^n$ be a continuous function and fix $\ep > 0$. Let $\bar f : 
    [a-\ep,b+\ep]$ be its continuous extension as done in Remark
    \ref{rmk:ContinuousExtension}. The following two statements are true:
    \begin{enumerate}
        \item $L(\bar f) = L(f)$.
        \item If $|t| \leq \ep$ and $g(s) = \bar f(s-t)$ for all $s \in [a,b]$,
        then $L(g) \leq L(\bar f) = L(f)$.
    \end{enumerate}
\end{lemma}
\begin{proof}
    We prove each statement separately.
    \begin{enumerate}
        \item Take a partition $P$ of $[a-\ep,b+\ep]$ with $N$ elements,
        such that there exists $0 < J < K < N$ such that $x_j = a$ and $x_k =
        b$. Then
        \begin{align*}
            &\sum_P||\bar f(x_i)-\bar f(x_{i-1})|| \\
            &=
            \sum_{i=1}^{J-1}||f(x_i)-f(x_{i-1})|| +
            \sum_{i=J}^{K}||f(x_i)-f(x_{i-1})|| \\
            &+ \sum_{i={K+1}}^{N}||f(x_i)-f(x_{i-1})|| \\
            &= \sum_{i=J}^K||f(x_i)-f(x_{i-1})|| \leq L(f),
        \end{align*}
        where the last inequality comes from the fact that $\{x_J, \dots, x_K\}$ is a
        partition of $[a,b]$.
        Therefore, by definition of the supremum $L(\bar f) \leq L(f)$. The inequality
        $L(f) \leq L(\bar f)$ holds trivially by noting that a partition of $[a,b]$
        can be extended to create a partition of $[a-\ep,a+\ep]$ and we are summing
        positive terms. Therefore $L(f) = L(\bar f)$.
        \item Let $|t| \leq \ep$ and consider a partition $P = \{x_0,\dots,x_N\}$ of
        $[a,b]$. Clearly $P- \{t\} = \{x_0-t,\dots,x_n-t\}$ could be considered
        as a subset of a partition of $[a-\ep,a+\ep]$. Therefore by constructing
        $P' = (P-\{t\})\cup\{a-\ep,b+\ep\}$ then
        \begin{align*}
            &\sum_P||g(x_i)-g(x_{i-1})|| \\
            &= \sum_P||\bar f(x_i-t)-\bar
            f(x_{i-1}-t)|| \\
            &= \sum_{P-\{t\}}||\bar f(y_i)-\bar f(y_{i-1})|| \\
            &\leq
            \sum_{(P-\{t\})\cup\{a-\ep,b+\ep\}}||f(y_i)-f(y_{i-1})|| \\
            &\leq
            \sup_{P \text{ partition of }[a-\ep,b+\ep]}\sum_P||f(y_i)-f(y_{i-1})||
            =L(\bar f).
        \end{align*}
        Since the supremum is the least upper bound, it follows that $L(g) \leq L(\bar f) = L(f)$.
    \end{enumerate}
\end{proof}

Now we are ready for the main result regarding the length of the generated mollified path being shorter or equal than the original.
\begin{theorem}\label{thm:Length}
   Let $f : [a,b] \to \R^n$ be a continuous function and let $\varphi$ be a nonnegative
   mollifier. Fixed $\ep > 0$ let $\bar f :[a-\ep,b+\ep] \to \R^n$ be the
   continuous extension as in remark \ref{rmk:ContinuousExtension}. Define 
   $F : [a,b] \to \R^n$ as $F = \bar f * \varphi_{\ep}$; then $L(F) \leq L(f)$.
\end{theorem}
\begin{proof}
    Take a partition $P$ of $[a,b]$, then
    \begin{align*}
        &\sum_P||F(x_i)-F(x_{i-1})|| \\
        &\underset{||\cdot||\text{ Jens. ineq}}{\leq}
        \sum_P\int_{[-\ep,\ep]}||\bar f(x_i-t)-\bar f(x_{i-1}-t)||\varphi_{\ep}(t)\mathrm{d}t \\
        &\underset{\text{linearity of integral}}{\leq}\int_{[-\ep,\ep]}\sum_P||\bar
        f(x_i-t)-\bar f(x_{i-1}-t)||\varphi_{\ep}(t)\mathrm{d}t \\
        &\leq \int_{[-\ep,\ep]}\sup_{P' \text{ part. of } [a,b]}\sum_{P'}||\bar
        f(x_i-t)-\bar f(x_{i-1}-t)||\varphi_{\ep}(t)\mathrm{d}t \\
        &\underset{\text{Lemma
        \ref{lem:LengthOfPaths}}}{\leq}\int_{[-\ep,\ep]}L(\bar
        f)\varphi_{\ep}(t)\mathrm{d}t =L(\bar f) = L(f),
    \end{align*}
    thus, by definition of the supremum, $L(F) \leq L(f)$.
\end{proof}
\begin{remark}\label{rmk:EndingAndInitialPoints}
Note that this does not imply that if $f$ is a geodesic between $f(a)$ and
$f(b)$, then $F$ is also a geodesic between these points. This is because $F(a)
\neq f(a)$ or $F(b) \neq f(b)$ may occur. While Theorem \ref{thm:ConvexHull} and
Remark \ref{rmk:ContinuousExtension} guarantee that $F([a,b]) = (\bar
f*\varphi_{\ep})([a,b]) \subset \co f([a,b])$, we cannot ensure that $F$ has the
same starting and ending points as $f$. What Theorem \ref{thm:Length}
establishes is that by considering the actual starting and ending points of $F$,
we can ensure that $L(F) \leq L(f)$. An example of this property is 
shown in Figure \ref{fig:MolliDifferentPointsAndLengths}. The original function,
which is a linear interpolation of three points in $\R^2$ and whose domain is
$[a,b]=[0,2]$, is extended to the domain $[-\ep,2+\ep]$ with $\ep  = 0.5$ using
the extension presented in remark \ref{rmk:ContinuousExtension}. As it can be
seen from each of its components, $F_{\ep}(a) \neq f(a)$ and $F_{\ep}(b) \neq f(b)$,
and clearly $L(F_{\ep}) \leq L(f)$ as Theorem \ref{thm:Length} states. Finally
note that, while $f_2$ can be considered a geodesic between the points $(0,0)$
and $(2,2)$ in $\R^2$ using the Euclidean norm, $(f_2 * \varphi_{\ep})$ is not a
geodesic between those two points.
\end{remark}

\begin{figure}[t]
    \centering
    \includegraphics[width=\linewidth]{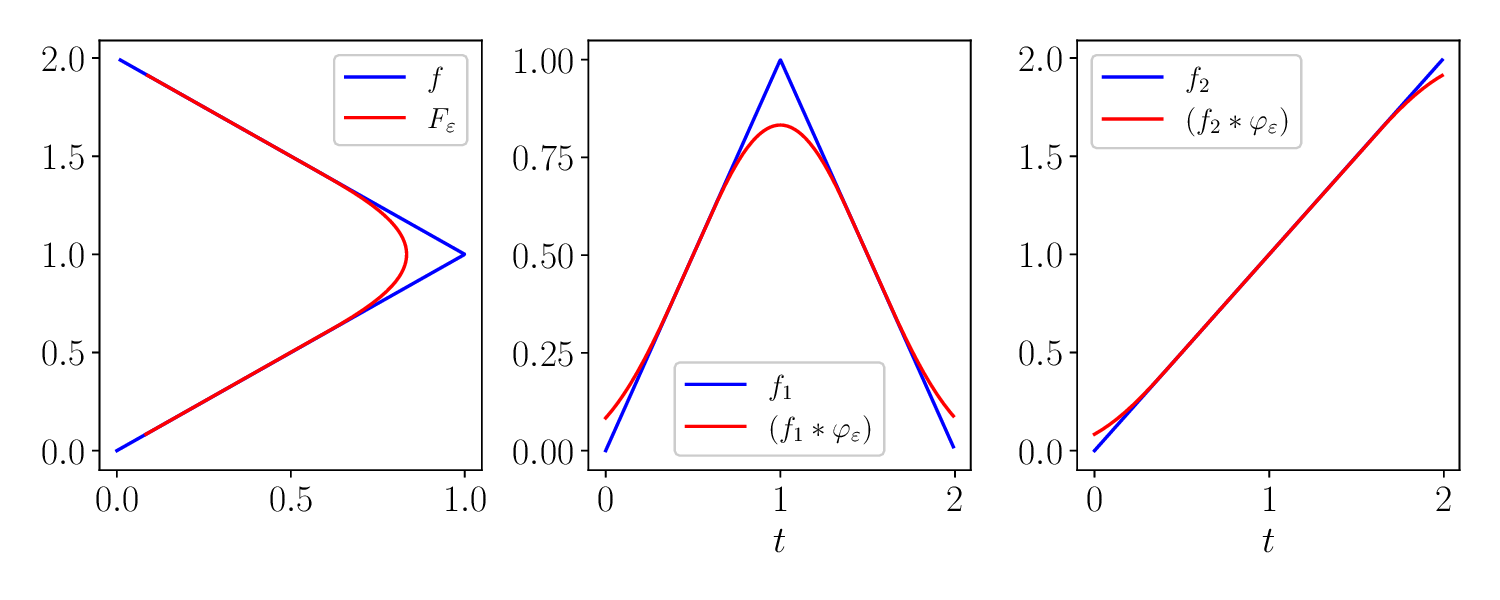}
    \caption{Visual representation of remark \ref{rmk:EndingAndInitialPoints}.
    The left picture represents the original path $f$ as a blue solid 
    line, and as a red solid line its mollification $F_{\ep} = (f_1 *
    \varphi_{\ep}, f_2 * \varphi_{\ep})$ where $\varphi$ is as in     \eqref{eq:OurMollifier} and $\ep = 0.5$. The middle
    picture represents the same information but for the first component of the
    function and its mollification, while the right picture represents the 
    same information but for the second component.}
    \label{fig:MolliDifferentPointsAndLengths}
\end{figure}

\subsection{The effect of reparametrization and mollification}
Suppose the desired path is encoded using a continuous function $f :
[a,b] \to \R^n$ and for a given $\ep > 0$ we consider its mollification with
parameter $\ep$. As we have seen, we first need to extend the function $f$ to $\bar{f} : [a-\ep,b+\ep]$ as is done in Remark \ref{rmk:ContinuousExtension} or
in Theorem \ref{thm:Length}. It may be of interest to reparametrize the curve
so it is normalized, i.e., to find a function $g : [-\ep,1+\ep] \to [a-\ep,b+\ep]$, and consider $\bar{f} \circ g * \varphi_{\ep} : [0,1] \to \R^n$, as the
desired mollified path. How does the parameter of 
the mollification change under these conditions? That is, does there
exist a $\eta = \eta(\ep) > 0$ such that
\begin{equation*}
    (\bar f \circ g * \varphi_{\ep}) ([0,1]) = (\bar f * \varphi_{\eta})([a,b])?
\end{equation*}

What we require is a continuous function $g:[-\ep,1+\ep] \to [a-\ep,b+\ep]$ that is strictly increasing and satisfies $g([0,1]) = [a,b]$. However, it is easy to see that such a function must be nonlinear. In the best-case scenario where such a function exists it is invertible, if it happens to be differentiable, we arrive at the following conclusions. Suppose $g : [-\ep,1+\ep] \to [a-\ep,b+\ep]$ is a continuously differentiable, increasing function that is nonlinear but satisfies $g([0,1]) = [a,b]$—as previously required. In this case, let $s \in [0,1]$ and applying the change of variables $v(t) = g(s-t)$,
we get
    \begin{align*}
        &(f*\varphi_{\ep})(s)=\int_{[-\ep,\ep]}f(g(s-t))\varphi_{\ep}(t)\mathrm{d}t \\
        &=
        \int_{[g(s-\ep),g(s+\ep)]}f(v)\varphi\left(\frac{s-g^{-1}(v)}{\ep}\right)\frac{1}{\ep
        g'(g^{-1}(v))}\mathrm{d}v \\
        &= \int_{[g(s-\ep), g(s+\ep)]}f(v)\varphi_{\ep}(s-g^{-1}(v))\frac{1}{g'(g^{-1}(v))}\mathrm{d}v.
    \end{align*}
Since $g$ is nonlinear, there is no straightforward way to solve for $v$ and
obtain a convolution-like expression with a single parameter in terms of $\ep$.
The effect to the reparametrization on $\ep$ may seem like an artificial question
to be posed. Nevertheless, note that
for a planar $f$ that is parametrized
in arc length, its curvature can be simply computed as $\kappa(s) = ||f''(s)||_2$,
with $s \in [0,L(f)]$.
Nevertheless, the arc-length parametrization is, in general, non-linear.
Therefore we have shown that we cannot find an upper
bound for the curvature that depends on the parameter $\ep$ for the mollified curve using
arc-length parametrization. Moreover, we add
that it can be shown, but it is not included in this work 
due to its cumbersome formulas, that if the mapping $g : [-\ep,1+\ep] \to
[a-\ep,b+\ep]$ is affine, continuous and increasing, then it is \textit{unique}, and
there is an expression relating $\eta$ and $\ep$, which can be easily found by
rudimentary computations. Nevertheless, it happens that
$(\bar{f} \circ g * \varphi_{\ep})([0,1]) \subset (\bar{f} *
\varphi_{\eta})([a,b])$, which implies that we do not generate the complete
mollified path after the reparametrization. Thus, even in the affine reparametrization
situation, the mollification does not behave well under the reparametrization of
curves.

\section{Curvature guarantees from a sequence of waypoints}
\label{sec: curvature}
In this section, we show how to systematically analyze the curvature of the generated path. In particular, we provide a formula to upper bound the curvature of the mollification of a sequence of 2D or 3D waypoints connected by straight line segments, i.e., via linear interpolation. First, we restrict ourselves to the simpler case of two segments. From now on, $\R^n$ denotes either $\R^2$ or $\R^3$.
\subsection{The case of three points forming two segments}
Suppose the desired path can be encoded using a parametric
function of the following form.
\begin{definition}[Two line segments function]\label{def:TwoLinesSegment}
 Let $P_0,P_1,P_2 \in \R^n$, and let $f : [0,2] \to \R^n$ be 
\begin{equation}\label{eq:TwoLineSegments}
    f(t) = \begin{cases}
        P_0 + (P_1-P_0)t, & t \in [0,1], \\
        P_1 + (P_2-P_1)(t-1), & t \in [1,2]
    \end{cases}.
\end{equation}
We call $f$ the two-line segments function. And we call
\begin{equation}\label{eq:TwoLinesSegmentsExtended}
    \bar{f}(t) = \begin{cases}
        P_0 + (P_1-P_0)t, & t \in (-\infty,1], \\
        P_1 + (P_2-P_1)(t-1), & t \in [1,\infty)
    \end{cases},
\end{equation}
the two-lines segment extended function. Note that 
$\bar{f}\mid_{[0,2]} = f$. \hfill $\square$
\end{definition}

We know by Proposition \ref{prop:MollifyIdentityIsIdentity} that
affine functions are invariant under mollification when
Assumption \ref{ass: mol} is met and the mollifier is an
even function, such as in Example \ref{example:OurMollifier}.
From Figure \ref{fig:MollifOfAbs} and Proposition \ref{prop:ConvexityThenMolliAbove}
it is clear that for the curve defined in Definition \ref{def:TwoLinesSegment}
it happens that $f(1) \neq F_{\ep}(1)$ for all $\ep > 0$. However,
it of special interesting to compute the intervals in which the
original and the mollified curve coincide, because in those
intervals there is no approximation, there is an equality
between the original and mollified curve. This is easily answered
in the next proposition, in which we show that the initial
and ending points of the original and mollified curve
also coincide.
\begin{proposition}\label{prop:EqualityInSets}
    Let $\bar{f}$ be as in \eqref{eq:TwoLinesSegmentsExtended}.
    Suppose $\ep \in (0,\frac{1}{2})$, define for $r \in \{1,2\}$,
    $V_r := [r-1+\ep,r-\ep]$
    and let $F_{\ep} := \bar{f} * \varphi_{\ep}$ with $\varphi$ even.
    Then $F_{\ep}|_{V}=\bar{f}|_V = f_V$,  $F_{\ep}(0) = f(0)$, and $F_{\ep}(2) =
    f(2)$. 
\end{proposition}

\begin{proof}
    Let $t \in V_r$, and note that for any $s \in [-\ep,\ep]$ it holds
    that $r-1 \leq t-\ep < t-s < t+\ep  \leq r$. Thus, if $r = 1$ the function to be mollified is the line
    $P_0 + (P_1-P_0)t$ for any $t \in V_r$ and if $r = 2$ the function
    to be mollified is the line $P_1 + (P_2-P_1)(t-1)$ for any 
    $t \in V_r$. Thus, by Proposition \ref{prop:MollifyIdentityIsIdentity}
    we get the desired result. The proof that the initial and ending points
    are the same it is carried in a similar fashion.
\end{proof}

It is straightforward to note that $f$ is differentiable
in $(0,2) \setminus \{1\}$ and $\bar{f}$ in $\R\setminus\{1\}$. 
Note that in both cases, the set of points on which the functions
are not differentiable form a set of measure zero, and the expressions
of their derivatives are constant functions. 
Let $t \in \R\setminus\{1\}$, then
\begin{equation*}
    \bar{f}'(t) = \begin{cases}
        P_1-P_0, & t \in (-\infty,1) \\
        P_2-P_1, & t \in (1,\infty)
    \end{cases}.
\end{equation*}
For $r \in \N$ we define $\PP_r := P_{r} - P_{r-1}$. Note that both $f$ and
$\bar{f}$ are continuous functions, hence locally integrable, and from Theorem
\ref{thm:PropertiesOfMollifying} we have that if $F_{\ep} := \bar{f} *
\varphi_{\ep}$ where $\varphi$ is any mollifier and $\ep > 0$ then $F_{\ep}' =
\bar{f} * \varphi_{\ep}'$. It is not difficult to prove that in this case,
$F_{\ep}' = \bar{f} * \varphi_{\ep}' = \bar{f}' * \varphi_{\ep}$ for any $\ep >
0$.

We can now exploit these results  to obtain a formula for the curvature, as well
as an upper bound.

\subsubsection{Computing the exact curvature}
We know by the discussion above that $F_{\ep}' =
(\bar{f}*\varphi'_{\ep}) = (\bar{f}' * \varphi_{\ep})$ everywhere. Consider a
mollifier $\varphi$ and let $\ep > 0$. Given $t \in \R \setminus \{1\}$, note
that
\begin{equation*}
    \bar{f}'(t) = \PP_1 \ind_{(-\infty,1)}(t) + \PP_2\ind_{(1,\infty)}(t),
\end{equation*}
hence for $ t \in \R$
\begin{align*}
    F_{\ep}'(t) &= \int_{\R}\varphi_{\ep}(t-s)\bar{f}'(s)\mathrm{d}s \\
    &= 
    \int_{(-\infty,1]}\varphi_{\ep}(t-s)\tilde{P}_1\mathrm{d}s + \int_{[1,\infty)}\varphi_{\ep}(t-s)\tilde{P}_2\mathrm{d}s \\
    &= \tilde{P}_1 \int_{(-\infty,1]}\varphi_{\ep}(t-s)\mathrm{d}s + \tilde{P}_2
    \int_{[1,\infty)}\varphi_{\ep}(t-s)\mathrm{d}s.
\end{align*}
Note that if $\Phi_{\ep} : \R \to \R$ is such that $\Phi_{\ep}' = \varphi_{\ep}$, then 
\begin{align*}
    \frac{d}{dt}\int_{[a,b]}\varphi_{\ep}(t-s)\mathrm{d}s &=
    \frac{d}{dt}\int_{[t-b,t-a]}\varphi_{\ep}(u)\mathrm{d}u \\
    &= \frac{d}{dt}(\Phi_{\ep}(t-b)-\Phi_{\ep}(t-a)) \\
    &=\varphi_{\ep}(t-b)-\varphi_{\ep}(t-a),
\end{align*}
thus
\begin{equation*}
    F_{\ep}''(t) = \varphi_{\ep}(t-1)(\PP_1-\PP_2).
\end{equation*}

Now define
\begin{align*}
    A_1(t) &:= \int_{(-\infty,1]}\varphi_{\ep}(t-s)\mathrm{d}s \\
    A_2(t) &:= \int_{[1,\infty)}\varphi_{\ep}(t-s)\mathrm{d}s,
\end{align*}
therefore, if $\kappa : \R \to \R$ is the curvature,

\begin{align*}
    \kappa(t) &= \frac{|| F_{\ep}''(t) \wedge F_{\ep}'(t)||_2}{||F_{\ep}'(t)||_2^3}
    \\ &= \frac{||\varphi_{\ep}(t-1)(\PP_1-\PP_2) \wedge
    (\tilde{P}_1 A_1(t) + \tilde{P}_2 A_2(t))||_2}{||F_{\ep}'(t)||_2^3} \\
    &=
    \varphi_{\ep}(t-1)|A_2(t)+A_1(t)|\frac{
    ||\tilde{P}_2\wedge \tilde{P}_1||_2}{||\tilde{P}_1A_1(t)+\tilde{P}_2A_2(t)||_2^3},
\end{align*}

and noting that, due to the properties of the mollifier,
$A_1(t)+A_2(t) = 1$ and $A_1(t),A_2(t) \geq 0$ for all
$t \in \R$, we have that 
\begin{equation}\label{eq:CurvatureClosedFormula}
    \kappa(t) = \varphi_{\ep}(t-1)\frac{
    ||\tilde{P}_2\wedge \tilde{P}_1||_2}{||\tilde{P}_1A_1(t)+\tilde{P}_2A_2(t)||_2^3}.
\end{equation}
Equation \eqref{eq:CurvatureClosedFormula} is an exact formula for the curvature
at each $t \in \R$.
\subsubsection{Upper bounding the curvature}
Note that $\varphi_{\ep}(t-1) \leq \frac{1}{\ep}||\varphi||_{\infty}$ for all
$t \in \R$. Moreover, it is clear that $F_{\ep}'(t)$ is the convex combination
of $\PP_1$ and $\PP_2$. Therefore
\begin{equation*}
    ||F_{\ep}'(t)||_2^2 \geq \min_{s\in[0,1]}||s\PP_1 + (1-s)\PP_2||_2^2 =: \min_{s\in[0,1]}g(s).
\end{equation*}
Note that $g$ is a differentiable convex function, so its minimum exists in the 
compact set $[0,1]$ and by the KKT conditions it is necessary and sufficient
to find an $\bar{s} \in [0,1]$ such that $g'(\bar{s}) = 0$. In this case
\begin{equation*}
    g'(\bar{s}) =0 \Longleftrightarrow \bar{s} = \frac{\dotProduct{\tilde{P}_2-\tilde{P}_1}{\tilde{P}_2}}{||\tilde{P}_2-\tilde{P}_1||_2^2}.
\end{equation*}
Since $g$ is positive  $||F_{\ep}||_2 \geq \sqrt{g(\bar{s})}$.
Also note that when differentiating, and making it equal to $0$ we are not constraining the values 
of $\bar{s}$. It may happen that $\bar{s} < 0$ or $\bar{s} > 0$. Nevertheless, since the function
$g$ is convex—in fact, strictly convex as long as $\PP_1 \neq \PP_2$—we know that  if the minimum of the unconstrained problem is not in the feasible set, i.e., $[0,1]$, then it  is at the boundaries of the feasible set. For this reason if it happens that
that $\bar{s} < 0$ then $g(0) = ||\tilde{P}_1||_2$ is the minimum value because 
$\bar{s} < 0$ is where the minimum occurs and the function is convex, while if $\bar{s}
> 1$ then $g(1) = ||\tilde{P}_2||_2$ is the minimum value.
Therefore,
\begin{align*}
\min_{s\in [0,1]}&||\tilde{P}_1s - \tilde{P}_2(1-s)||_2
 \\
&=\begin{cases}
        \left|\left|\tilde{P}_1
    \bar{s} +
    \left(1-\bar{s}\right)\tilde{P}_2\right|\right|_2,
    & 0 \leq \bar{s} \leq 1 \\
    \min\{||\tilde{P}_1||_2, ||\tilde{P}_2||_2\}, & \text{ otherwise }
    \end{cases}
\end{align*}
From which it follows that if 
\begin{align*}
    M(\tilde{P}_1,\tilde{P}_2)
    := \begin{cases}
        \frac{1}{\left|\left|\tilde{P}_1
    \bar{s} +
    \left(1-\bar{s}\right)\tilde{P}_2\right|\right|_2^3},
    & 0 \leq \bar{s} \leq 1 \\
    \max\left\{\frac{1}{||\tilde{P}_1||_2^3}, \frac{1}{||\tilde{P}_2||_2^3}\right\}, & \text{ otherwise }
    \end{cases},
\end{align*}
then $\frac{1}{||F_{\ep}'(t)||_2^3} \leq M(\PP_1,\PP_2)$, and using  \eqref{eq:CurvatureClosedFormula} we
arrive at
\begin{equation}\label{eq:UpperBoundCurvature}
    \kappa(t) \leq \frac{1}{\ep}||\varphi||_{\infty}||\PP_1\wedge\PP_2||_2M(\PP_1,\PP_2), \quad \forall t \in \R.
\end{equation}
That is, we have found an upper bound on the curvature for two segments that is independent of $t$. Figures \ref{fig:FirstUpperBoundCurvature} and \ref{fig:SecondUpperBoundCurvature} illustrate this upper bound for two different curves. Note that when the segments have similar lengths, the upper bound equals the maximum curvature—making it the tightest possible bound. However, when one segment is significantly longer than the other, the zone of maximum curvature shifts from $t = 1$ due to the mollification process. In any case, we can confidently assert that this upper bound is a good approximation of the maximum curvature, and in many cases optimal.
\begin{figure}
    \centering
    \includegraphics[width=\linewidth]{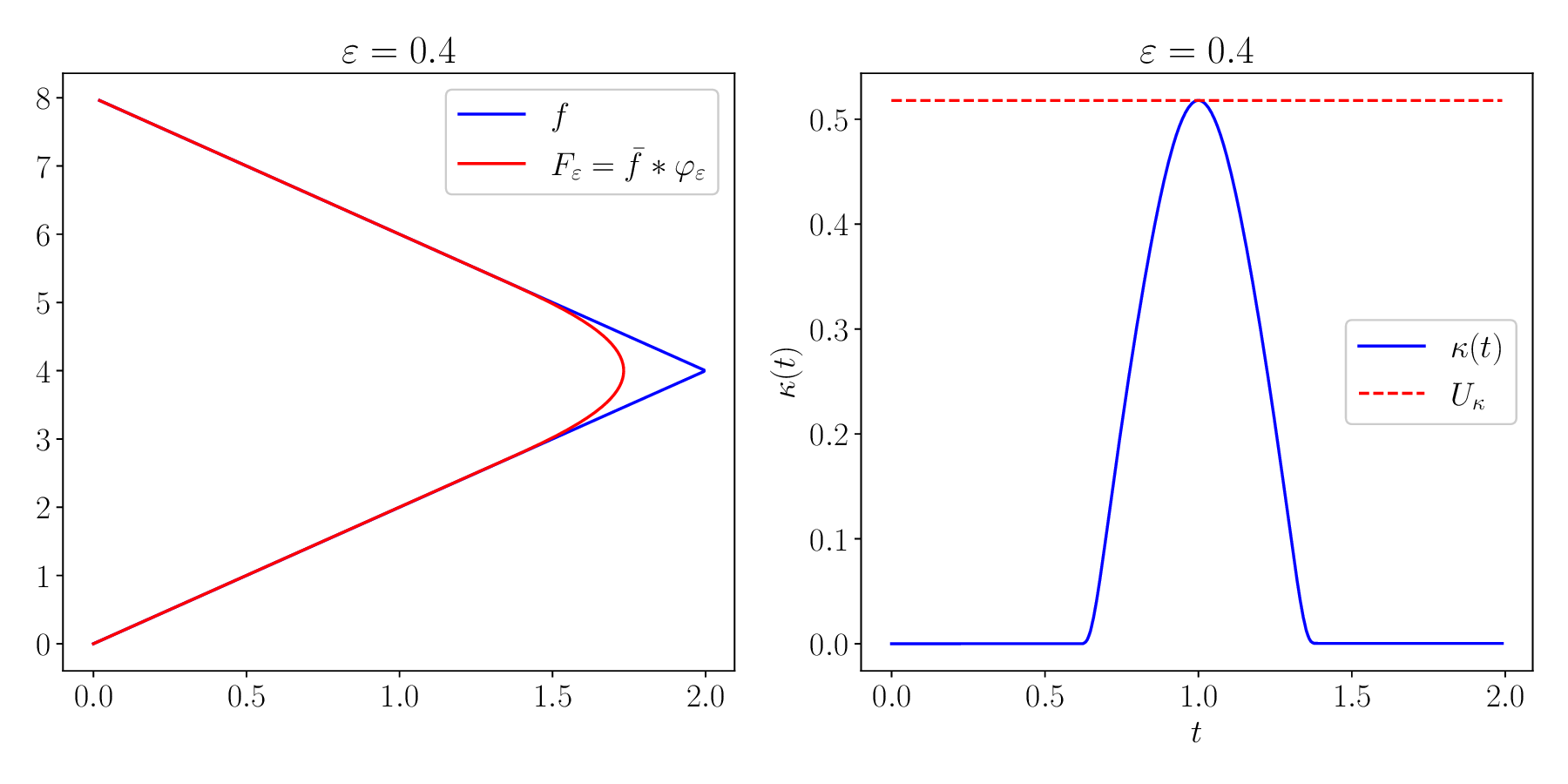}
    \caption{Left plot represents of a three-point-two-segment function and its mollification. Right plot represents
    the curvature of the function and its 
    upper bound $U_{\kappa}$ which is 
    the right hand side of
    \eqref{eq:UpperBoundCurvature}. In this
    case the zone of maximum curvature 
    corresponds to a point really close (or equal) to $t= 1$.
     We have used the mollifier presented in Example \ref{example:OurMollifier}.}
    \label{fig:FirstUpperBoundCurvature}
\end{figure}
\begin{figure}
    \centering
    \includegraphics[width=\linewidth]{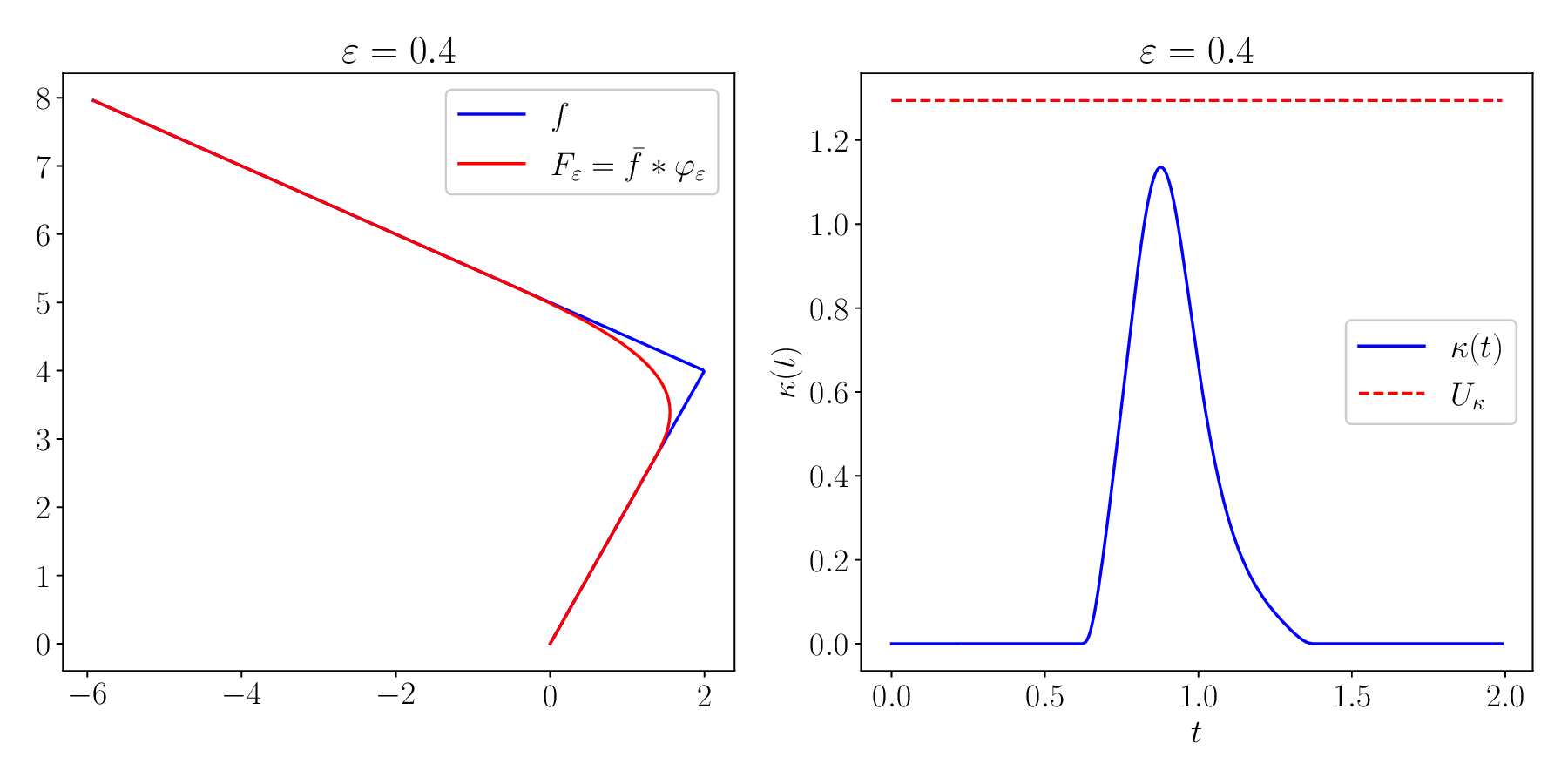}
    \caption{Left plot represents of a three-point-two-segment function and its mollification. Right plot represents
    the curvature of the function and its 
    upper bound $U_{\kappa}$ which is 
    the right hand side of
    \eqref{eq:UpperBoundCurvature}. In this 
    case the zone of maximum curvature does not correspond to the join of the 
    two segments. We have used the mollifier presented in Example \ref{example:OurMollifier}.}
    \label{fig:SecondUpperBoundCurvature}
\end{figure}

\subsection{The general case}
We are going to present how, the natural generalization of the previous computations
to a $p >2$ segments curve, gives, in general, a worse result than considering
the curvature of each pair of segments locally, and then choosing the most restrictive
$\ep_{max} > 0$ so that all curvatures are constraint under this $\ep_{max} >
0$. Nevertheless, due to the cumbersome and rudimentary computations, we just
present the result.
The natural lower bound we arrive to the upper bound
of the function $\kappa$ for the $p$ segments curve can be found
to be
\begin{align*}
    \kappa \leq
    \frac{2}{\ep}\frac{||\varphi||_{\infty}}{||\proj_{\co(S)}(0)||_2^3}\sum_{j=1}^{p}\sum_{i=j+1}^{p+1}&||\PP_i \wedge \PP_j||_2, \\ &\text{ as long as}
    \quad 0 \notin \co(S).
\end{align*}
where $S := \{\PP_i\}_{i=1}^{p}$, $\PP_i := P_i-P_{i-1}$, $i \in \{1,\dots,p+1\}$ and $\proj_{\co(S)}(0)$ is the unique
element $s \in \R^n$ such that $d(0,S) = ||s||_2$, that is, the projection of $0$
onto $\co(S)$. 
Nevertheless, from this equation one can see that if several points are collinear,
then $0 \in \co(S)$, increasing the upper bound, which is contradictory to the
fact that the curvature shall decrease. Thus, we conclude that, the natural
generalization of the three point two segment approach cannot be used for the
$p+1$ points $p$ segments approach.

We propose the following methodology. Suppose we have $p \in \N$ segments with $p \geq 2$. Using  \eqref{eq:UpperBoundCurvature} and given a maximum curvature $\kappa_{\max} > 0$, we can compute for each pair of consecutive segments its respective\footnote{The index $i$ ranges from the first pair to the last pair of consecutive segments.} $\ep_{i} > 0$ such that, under the three-point-two-segment approximation, their curvatures are upper bounded. If $\ep_{i} < \frac{1}{2}$ for all $i$, then \eqref{eq:UpperBoundCurvature} is exact, because only the two segments used for computing $\ep_i$ contribute to the mollification at the junction point. In this case, take $\ep = \max_{i}\ep_i$, which is valid because the only dependence on $\ep_i$ in the right-hand side of \eqref{eq:UpperBoundCurvature} is through $\frac{1}{\ep_i}$; hence, $\ep$ satisfies the bound for each pair of segments. If $\ep_{i} > \frac{1}{2}$ for some $i$, then \eqref{eq:UpperBoundCurvature} becomes an approximation. In this case, one can either accept an admissible error or use $\ep = \max_{i}\ep_i$ as an initial condition for an optimization algorithm that seeks the minimum $\ep > 0$ that upper bounds the curvature. In either case, \eqref{eq:UpperBoundCurvature} is a powerful, computationally inexpensive tool that can be used to either compute an exact upper bound for the complete trajectory or reduce computation time in an optimization algorithm.

\section{Numerical validations and real experiments}
\label{sec: exp}
To demonstrate the effectiveness of our path generation approach, we first 
show a comparison between the path generated by different spline methods and
the one by mollification. We also present both numerical and experimental results for path following of a mollified path by a unicycle vehicle. Specifically, we employ the Singularity-Free Guiding Vector Fields (SF-GVF) path following algorithm \cite{WeijiaGVF, weijiaarticlegvf}. In brief, SF-GVF takes a parametric path $f\in C^2(\R,\R^n)$ as input and constructs a vector field $\chi \in C^2(\R^n,\R^n)$ whose flow traces the mollified path.

\subsection{Comparison with traditional interpolation methods}

We recall that mollification can generate paths from more general inputs than those for which interpolation is possible, such as the input that considers only $x\in\mathbb{Q}$ in Remark \ref{rem: Q}. Nonetheless, to compare traditional interpolation methods with our path generation via mollification, consider a collection of points in $\mathbb{R}^2$ that are linearly interpolated, and their $C^2$ cubic splines, B-splines and quintic Hermite interpolation splines, i.e, polynomial splines that are twice continuously differentiable and which are computed to specifically pass through the pre-defined collection of points. We compare the generated paths between
the spline approaches and the proposed mollification in Figure \ref{fig:ComparisonWithSplines}.
\begin{figure}
    \centering
    \includegraphics[width=\linewidth]{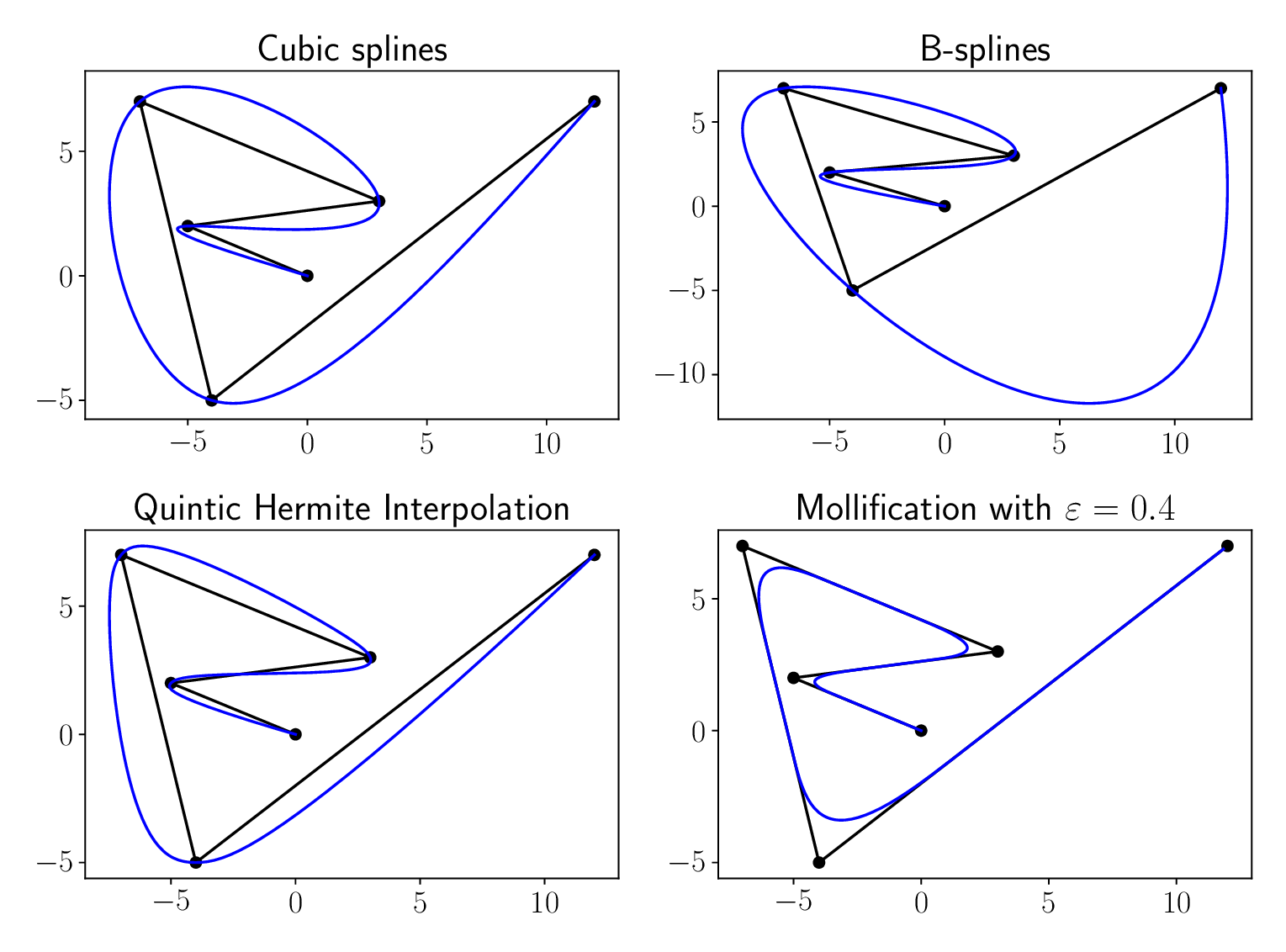}
    \caption{Comparison of different approach for path generation. The blue
    line represents the path generated from the linearly interpolated sequence
    of points shown in black. The mollifier used is the one presented in
    Example \ref{example:OurMollifier} with $\ep = 0.4$.}
    \label{fig:ComparisonWithSplines}
\end{figure}
As it can be seen, the mollification approach is the method that resembles the 
original function the most. Moreover, due to Proposition \ref{prop:EqualityInSets}
we know \textit{exactly} in which intervals the original function and the mollified
one coincide. Clearly, this is at the cost of not intersecting the collection of points,
except the initial and ending ones. Nevertheless, note that numerically speaking,
the mollification method is by far the simplest one. Indeed, for $C^2$ continuity
there are several conditions that the coefficients of the spline polynomials must met,
and which are extremely sensible to changes in the path. That is, a small change
in the original path, i.e., a small change in a point of the collection, can generate
a big change in the generated path. This is not something that happens in mollification,
since by Proposition \ref{prop:EqualityInSets} and the results of Section \ref{sec: res}
we know exactly in which sets the original path and its mollification coincide, where
is the mollified curve enclosed, as
well as the geometric properties preserved. Finally, we want
to remark an important fact. The polynomials spline here presented are \textit{at most}
$C^2$ continuous at the collection of points, and its maximum curvature is not easy to compute between two 
segments nor the complete spline. For this purpose an
optimization approach is often carried as mentioned in the introduction.
In the mollified case, with just a single
operation, $C^{\infty}$ continuity is obtained, and its derivatives
are easy to compute thanks to Theorem \ref{thm:PropertiesOfMollifying}. We have 
also provided, for
this specific example, an analytical and simple to compute curvature upper bound in (\ref{eq:UpperBoundCurvature}). This implies that mollified
path can be computed or changed online in low-cost platforms.

\subsection{Usage with a path following algorithm}
For the path-following part, we are going to consider the so called ``heart'' function as our input path. Define the function
    \begin{equation*}
        t\in[0,2\pi) \to r(t) = 2 - 2\sin(t) + \sin(t)
        \frac{\sqrt{|\cos(t)|}}{\sin(t)+1.4}.
    \end{equation*}
For $t \in [0,2\pi)$ let $f_1(t) = r(t)\cos(t)$ and $f_2(t) = r(t)\sin(t)$, and we call $f := (f_1,f_2)$  the ``heart'' path. Note that the ``heart'' path is continuous but not differentiable; therefore, it cannot be used for the path following algorithm SF-GVF. We solve this issue by approximating the function using mollifiers. Let $\varphi$ the mollifier presented in Example \ref{example:OurMollifier} and let $\ep_1,\ep_2 > 0$ be real numbers. We then mollify the ``heart'' path $F$ is defined as follows
\begin{equation*}
    F = (F_1, F_2) := (f_1 * \varphi_{\ep_1},
    f_2*\varphi_{\ep_2}).
\end{equation*}

\begin{figure}
    \centering
    \includegraphics[width=\linewidth]{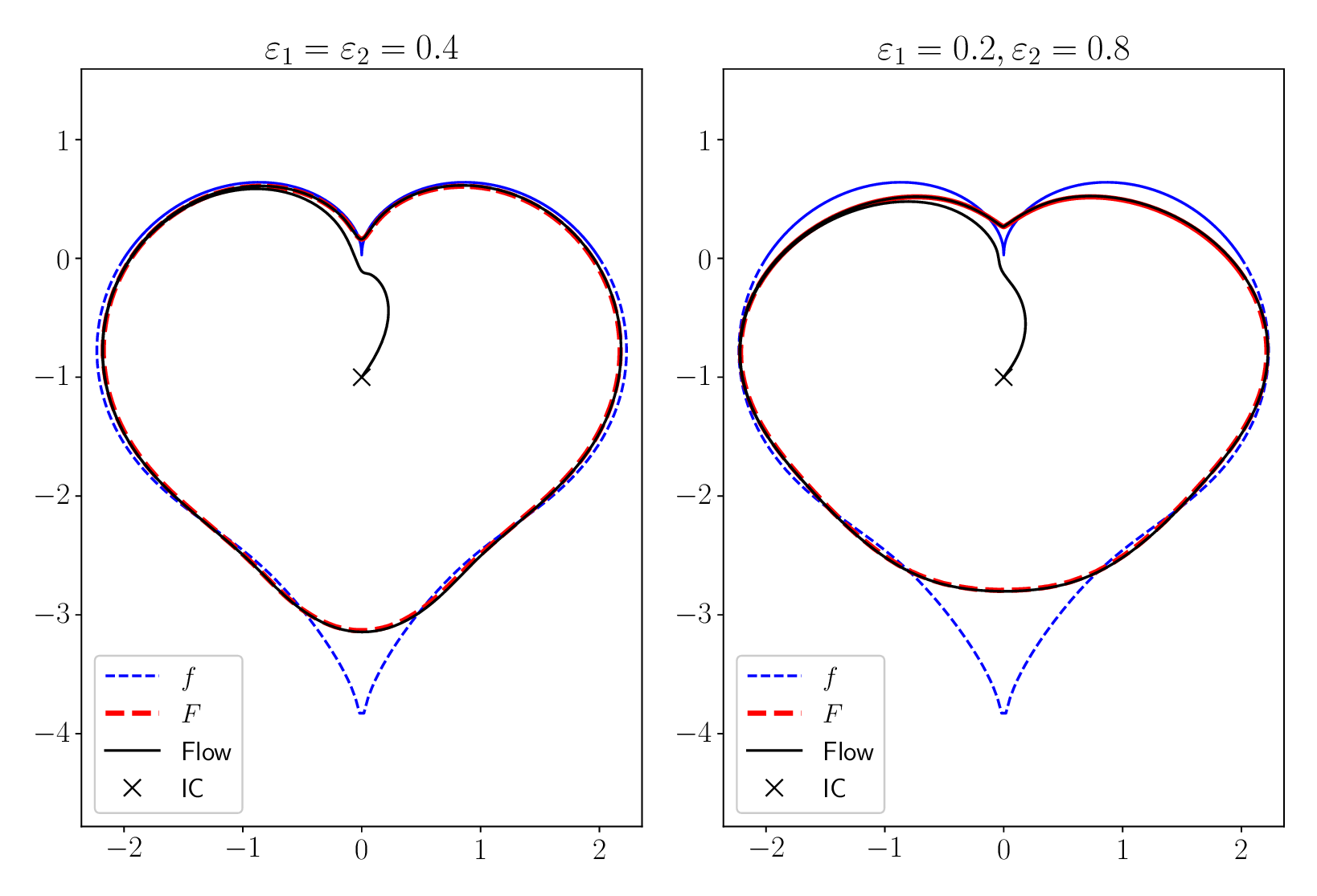}
    \caption{Representation of the numerical simulation. In both pictures the original path $f$,
    is shown as a blue dotted line, while as a red dotted
    line the mollified trajectory $F$ for $\ep_1=\ep_2=0.4$ in the left picture, and 
    for $\ep_1 = 0.2$ and $\ep_2 = 0.8$ in the right picture. The black solid
    line represents the (now smooth) flow generated by the guiding vector field according
    to \cite{weijiaarticlegvf} starting from an arbitrary initial condition (IC).}
    \label{fig:GVFHeartMolliComplete}
\end{figure}

A numerical simulation of the vehicle under SF-GVF using $\ep_1 = \ep_2 = 0.4$ for the ``heart'' path is shown in the left plot of Figure \ref{fig:GVFHeartMolliComplete}. The vehicle's trajectory indicates convergence to the desired mollified path. Moreover, as $\ep \to 0$, the trajectory approaches the original path more closely. Note that the mollified path lies inside the original path, as predicted by Theorem \ref{thm:ConvexHull}. In practical terms, this means that the ``heart'' function can now be used with SF-GVF, extending the applicability of this path following algorithm. Clearly, to improve convergence to the original path, we can reduce both $\ep_1$ and $\ep_2$, since by Theorem \ref{thm:PropertiesOfMollifying} we have uniform convergence on compact sets as $\ep_1,\ep_2 \to 0$. However, to demonstrate the flexibility of the approach, we also consider the case $\ep_1 = 0.2$ and $\ep_2 = 0.8$, whose simulation is presented in the right plot of Figure \ref{fig:GVFHeartMolliComplete}. Note that in the first component, the mollified curve is better adjusted to the original curve, while in the second component, the opposite occurs. This results from $\ep_2$ being four times larger than $\ep_1$. Indeed, the values of $\ep_1$ and $\ep_2$ can be constrained by the vehicle's dynamics. This is a key advantage of the method: by simply adjusting these parameters, we can ensure that the vehicle follows the curve within its dynamic limits, thereby avoiding issues related to reconverging to the path. Moreover, numerical computations show that the length of the original path with respect to the $\ell_1$ norm is (in arbitrary units) $25.58$, while the mollified curve has length $23.16$ in the left case and $21.74$ in the right case of Figure \ref{fig:GVFHeartMolliComplete}, as predicted by Theorem \ref{thm:Length}. The same conclusions from Theorem \ref{thm:Length} can be verified with respect to any other arbitrary $\ell_p$ norm.

\begin{figure}
    \centering
    \includegraphics[width=\linewidth]{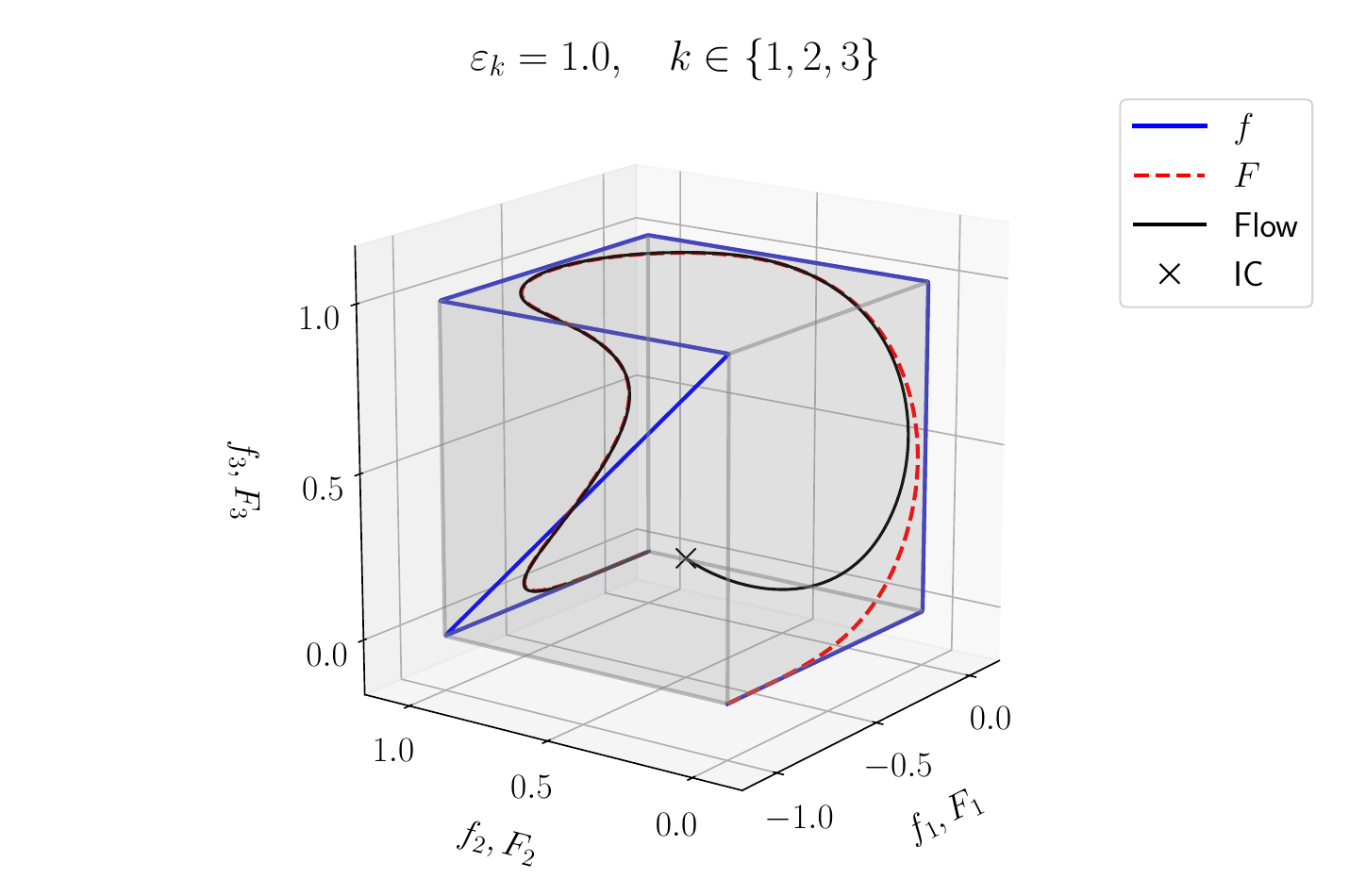}
    \caption{Representation of the three dimensional numerical simulation.
    The original path $f$ is represented as a solid blue line,
    while the mollified trajectory $F = (F_k)_{k=1}^3 = (f_k *
    \varphi_{\ep_k})_{k=1}^3$ is shown as a dashed red line, with $\varphi$ as in
    example \ref{example:OurMollifier} and $\ep_k = 1$ for $k \in \{1,2,3\}$. The black solid
    line represents the (now smooth) flow generated by the guiding vector field according to \cite{weijiaarticlegvf} starting from an arbitrary initial condition (IC).}
    \label{fig:3DMollification}
\end{figure}

\subsection{Three dimensional path mollification}
Finally, for completeness, a numerical simulation of a 3D mollified path is presented in Figure \ref{fig:3DMollification}. The notation used is identical to that in the previous numerical simulations. The original path $f$ is constructed via linear interpolation between a sequence of vertices/waypoints of a three-dimensional cube. As can be seen, $f$ is non-differentiable at these vertices. In contrast, the mollified function $F$ provides a smooth approximation that can be effectively employed in SF-GVF, as illustrated in Figure \ref{fig:3DMollification}. Indeed, the flow of the guiding vector field converges to the mollified trajectory. Moreover, Theorem \ref{thm:Length} can also be validated numerically. In this case, the length of the original path in the $\ell_2$ norm is (in arbitrary units) $7.38$, while the length of the mollified path is $5.48$. An important feature of this approach is its scalability: any trajectory of the form $g : \R \to \R^n$ can be mollified component-wise, producing a sufficiently smooth curve that can be further adapted to satisfy a variety of constraints.

\subsection{Experimental results}
Before presenting the experimental results, we introduce the software and
hardware platforms used in the experiments. We also provide the necessary links
to the developed software so that any interested reader can replicate these
experiments.

\subsubsection{Experimental and Software Platform}
Our experimental platform, shown in Figure \ref{fig:Rover-Hardware}, is a rover modeled as a unicycle, built around a Matek F765-Wing autopilot with an STM32 microcontroller, integrated IMUs, and support for GNSS, compass, and radio receivers---specifically a Matek M10Q-5883, a Futaba 7008SB receiver, and a Zigbee Xbee telemetry radio. The entire system runs on the open-source Paparazzi UAV framework \cite{paparazzo}, which handles autonomous operation, telemetry, and real-time communication with the Ground Control Station (GCS). Through the GCS, the user can issue high-level commands and adjust waypoints on the fly, while the onboard microcontroller recomputes the mollified path in real time whenever a point or parameter is modified, as illustrated in Figure \ref{fig:GCSScreenShot}.
\begin{figure}
    \centering
    \includegraphics[width=\linewidth]{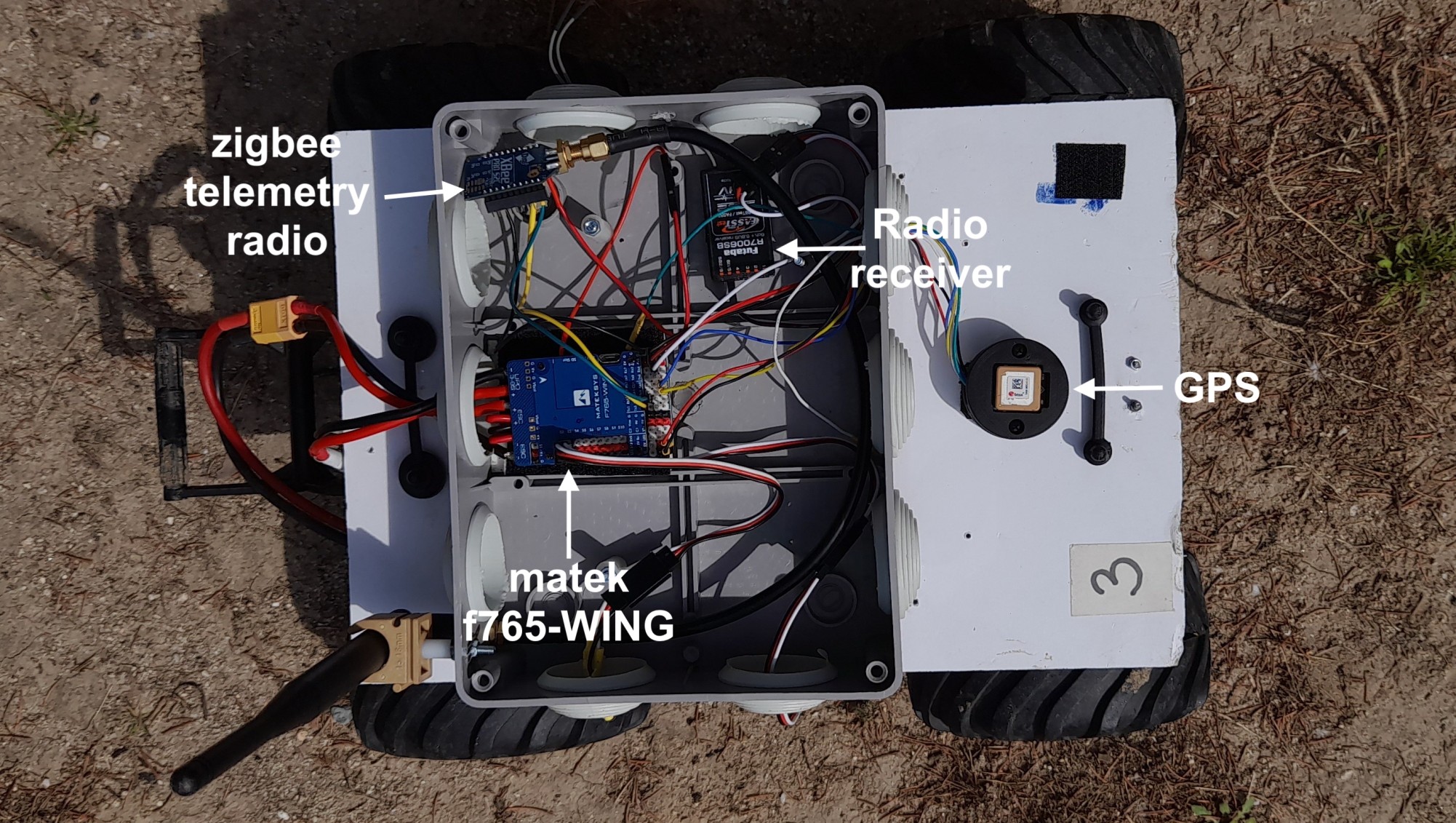}
    \caption{Rover vehicle and hardware used during the experiments.}
    \label{fig:Rover-Hardware}
\end{figure}

\begin{figure}
    \centering
    \includegraphics[width=\linewidth]{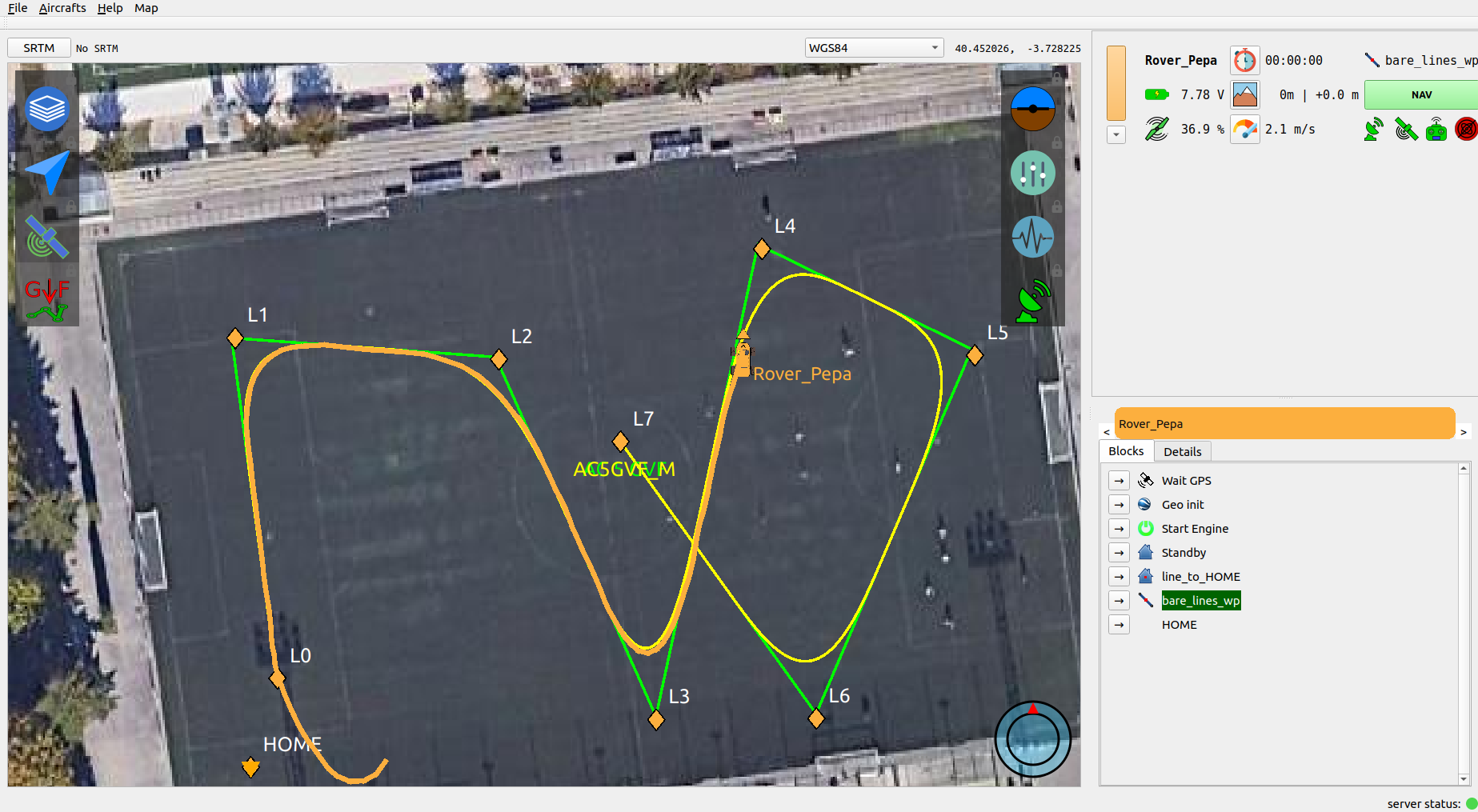}
    \caption{Capture of an experiment using Paparazzi GCS. The original
    trajectory $f$ is shown in green created by linearly interpolating the
    points $L_i \in \R^2$, $i \in \{0,\dots,7\}$. The yellow curve represents
    the mollification of the original trajectory $F=(f_1 * \varphi_{\ep}, f_2 *
    \varphi_{\ep})$ where $\varphi$ is as in Example \ref{example:OurMollifier}
    and $\ep = 0.5$. The orange line represents the trajectory described by the vehicle.}
    \label{fig:GCSScreenShot}
\end{figure}

\subsubsection{Experimental data}
We show our logs in Figure \ref{fig:Experiments}. We created a continuous but non-differentiable path by linearly interpolating points in $\R^2$ as it is done in Section \ref{sec: curvature}. As noted above, any Paparazzi user can create such trajectories by simply moving points in the ground control station before or during the experiment in real time. The experiment uses a curve similar to that in Figure \ref{fig:GCSScreenShot}.


\begin{figure}
    \centering
    \includegraphics[width=1\columnwidth]{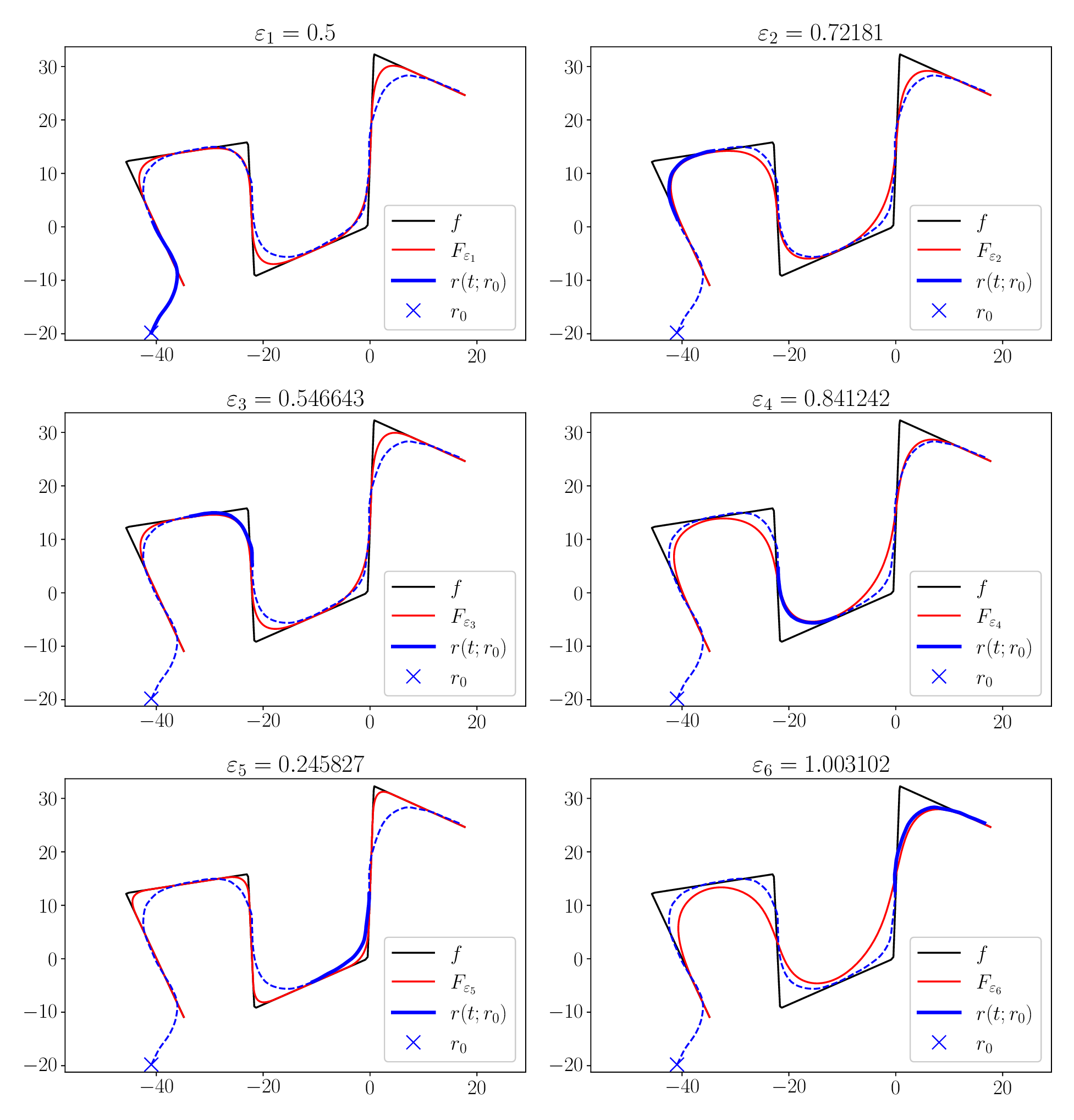}
    \caption{Real experiment of path following of a mollified non-differentiable path with a rover vehicle. The original non-differentiable path $f$ is in black, the mollified family
        trajectories $\{F_{\ep_i}\}_{i=1}^{6}$ to be followed at each \emph{stage} are in solid red, its initial 
        point $r_0$ is denoted as the blue cross, the position of the vehicle for its corresponding $\ep_i>0$ at times $t \in [t_{i-1},t_{i}]$ in a thick solid blue line. Finally, the complete
        trajectory of the vehicle $r(t;r_0)$ throughout all the stages is the dashed blue line.}
    \label{fig:Experiments}
\end{figure}


The experimental objective is as follows: Given a desired non-differentiable path $f$, mollify it with parameter $\ep > 0$ to obtain $F_{\ep} = (f * \varphi_{\ep})$, where $\ep$ is a function of vehicle speed limited by maximum allowed curvature. To demonstrate potential applications, we use a linear relationship between maximum allowed curvature and speed. For speed $v > 0$, with $R_{min}$, $R_{max}$, and $v_{max}$ denoting minimum radius, maximum radius, and maximum speed respectively, the allowed radius of curvature is $R(v) := R_{min} + \frac{v}{v_{max}}(R_{max}-R_{min})$. Note this is merely illustrative; curvature and speed need not be linearly related in practice. We are going to consider a curve consisting 
on seven two dimensional points which are linearly interpolated, as shown in Figure \ref{fig:Experiments}.

The adaptation of the curve to vehicle dynamics operates as follows: at initial time $t_0$, we set $\ep_1 = 0.5$ to generate the first mollified curve $F_{\ep_1}=f * \varphi_{\ep_1}$, where $f$ is the original curve and $\varphi$ is as in Example \ref{example:OurMollifier}. As the vehicle advances and reaches each segment midpoint at times $\{t_i\}_{i=1}^{6}$, the speed is measured and used to compute the maximum allowed curvature, from which the minimum permissible $\ep_{i+1}$ is determined via \eqref{eq:UpperBoundCurvature}. This generates a family of six parameters $\{\ep_i\}_{i=1}^{6}$ and corresponding mollified curves $\{F_{\ep_i}\}_{i=1}^{6}$ that dynamically adapt to vehicle constraints. As shown in Figure \ref{fig:Experiments}, the resulting paths remain close to the original trajectory where curvature permits, e.g., segments four and five, while strongly constraining the curve when necessary, e.g., the last two segments. All theoretical guarantees hold: the mollified curves lie within the convex hull of the original path (Theorem \ref{thm:ConvexHull}), and parameters are sufficiently small that Propositions \ref{prop:LocalConvexity} and \ref{prop:MolliAboveFConvexityLocally} apply. These experiments validate both the theoretical solution of Problem \ref{prob:RegularizationProblem} and the practical viability of our approach on real, affordable hardware.

\section{Conclusions}
\label{sec: con}

In this work, we addressed the problem of efficient path generation to make non-suitable curves, such as linear interpolations from waypoint collections, suitable for path following and trajectory tracking algorithms via mollification. The mollification can be adjusted so that the mollified trajectory approximates the original trajectory arbitrarily closely on compact sets while being completely smooth. Additionally, properties such as convexity, concavity, monotonicity, and quasiconvexity are preserved under mollification, with local versions also preserved for sufficiently small mollification parameters.

We validated the approach through numerical simulations using Singularity-Free Guiding Vector Fields as a path following algorithm, applying mollification to the ``heart'' path and a 3D trajectory while examining the effects of different parameter values. Finally, experiments on rovers demonstrated the viability of the approach and the ability to tune mollified paths to match vehicle dynamics. While the original trajectory may not be physically realizable, the mollified trajectory with appropriate parameters can be successfully followed. This confirms that our results have both theoretical significance and practical value for autonomous vehicles, industrial robotics, and any engineering application requiring fast and rigorous function approximation.







\bibliographystyle{IEEEtran}
\bibliography{biblio}

\begin{IEEEbiography}[{\includegraphics[width=1in,height=1.25in,clip,keepaspectratio]{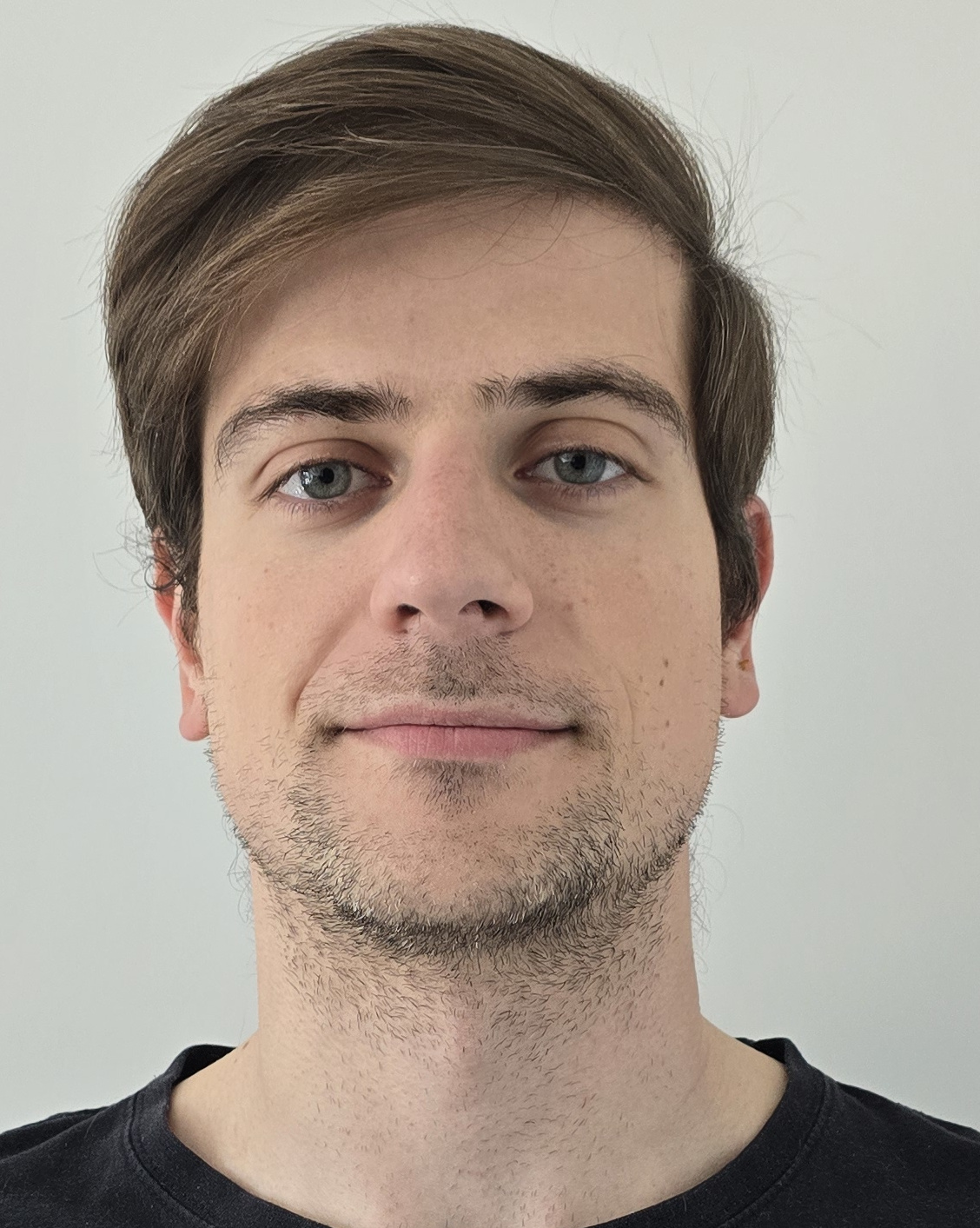}}]{Alfredo
 González-Calvin} earned his Electrical Engineering degree from the Complutense 
University of Madrid in 2023. He earned his two Master's degrees from UNED 
University: one in Control and Systems Engineering in 2024 and another in 
Mathematics in 2025. Currently, he is pursuing a Ph.D. in Physics, focusing on 
path planning and mathematical applications in robotics and engineering.
\end{IEEEbiography}

\begin{IEEEbiography}[{\includegraphics[width=1in,height=1.25in,clip,keepaspectratio]{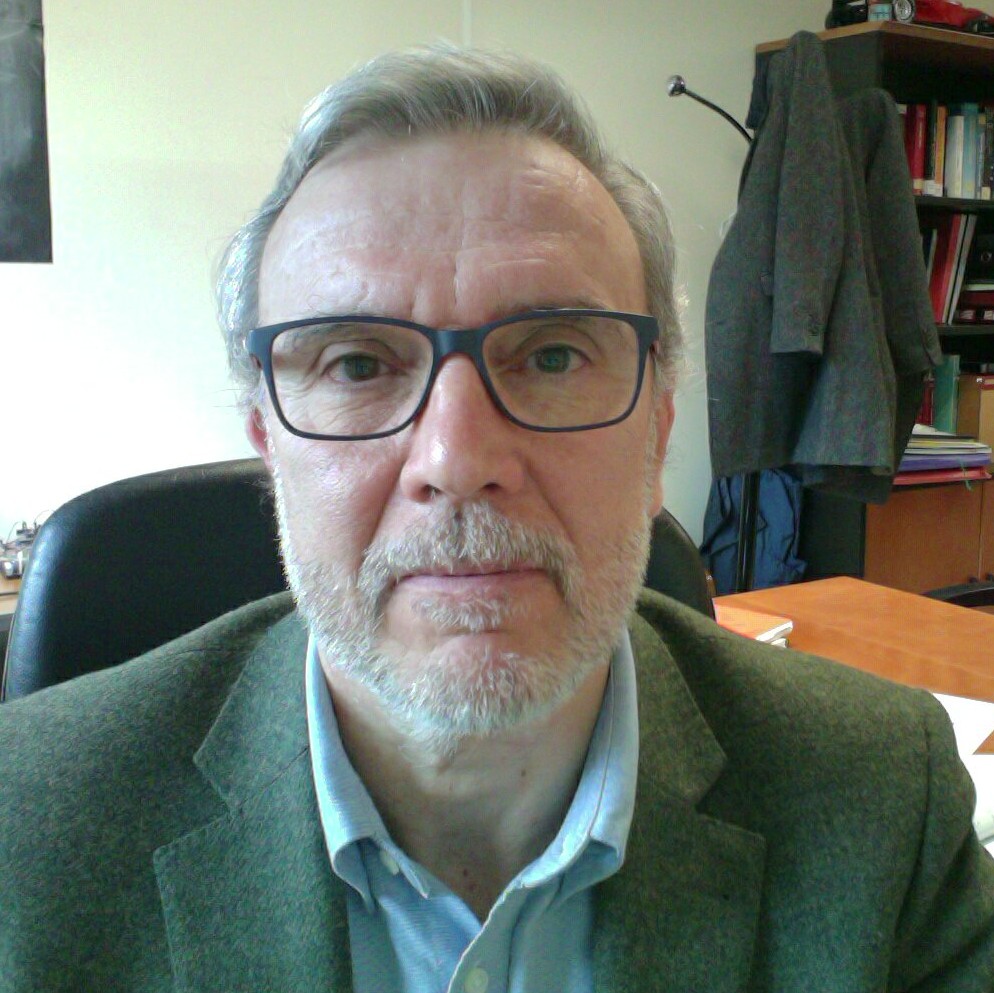}}]{Juan
 Jimenez} graduated in Physics from the Universidad Autónoma de Madrid (Spain) 
in 1986 and earned his Ph.D. in Systems Control in 1999 from the Universidad 
Nacional de Educación a Distancia (Spain). Since 2015, he has been an Associate 
Professor in the Department of Computer Architecture, Systems Engineering, and 
Automation at the Universidad Complutense de Madrid. His research interests 
include distributed control and cooperative control in multi-agent systems, 
with a particular focus on applications to autonomous vehicles.
\end{IEEEbiography}

\begin{IEEEbiography}[{\includegraphics[width=1in,height=1.25in,clip,keepaspectratio]{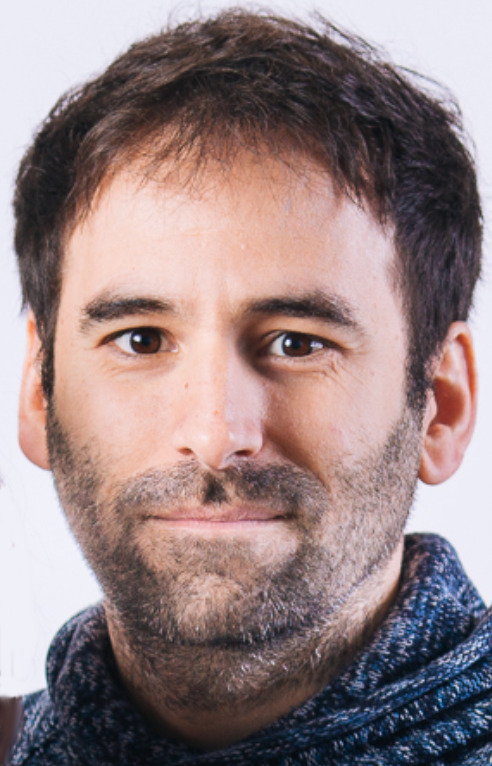}}]{Hector
 Garcia de Marina} (Member IEEE) received the Ph.D. degree in systems and 
control from the University of Groningen, The Netherlands, in 2016. He was a 
Postdoctoral Research Associate with the Ecole Nationale de l’viation Civile, 
Toulouse, France, and an Assistant Professor with the Unmanned Aerial Systems 
Center, University of Southern Denmark, Odense, Denmark. Since 2022, he has 
been a Ramón y Cajal Researcher with the Department of Computer Engineering, 
Automation and Robotics, and with CITIC, Universidad de Granada, Spain. He is 
the recipient of an ERC Starting Grant and was Associate Editor for IEEE 
Transactions on Robotics for four years. His current research interests include 
multiagent systems and the design of guidance navigation and control systems 
for autonomous vehicles.
\end{IEEEbiography}

\end{document}